\documentclass[11pt]{article}

\usepackage{amsmath,textcomp,amssymb,graphicx,amsthm,enumerate}
\usepackage{graphicx}
\usepackage{enumerate}
\usepackage{natbib}
\usepackage{xcolor}
\usepackage{mathtools}
\usepackage{ stmaryrd }
\usepackage{ enumitem}

\newtheorem{theoremM}{Theorem}
\newtheorem{corollaryM}{Corollary}
\newtheorem{theorem}{Theorem}[section]
\newtheorem{lemma}{Lemma}[section]
\newtheorem{corollary}{Corollary}[section]
\newtheorem{example}{Example}[section]
\theoremstyle{definition}
\newtheorem{definition}{Definition}[section]

\usepackage[affil-it]{authblk}

 \setlist[itemize]{leftmargin=*}
 \setlist[enumerate]{leftmargin=*}
 
\AtBeginDocument{}
\newenvironment{remark}[1][Remark:]{\begin{trivlist}
\item[\hskip \labelsep {\bfseries #1}]}{\end{trivlist}}

\newcommand{\D}{{\mathbf D}}

\newcommand{\w}{{\mathbf w}}
\newcommand{\Dset}{{\mathcal D}}
\newcommand{\Obj}{L}
\newcommand{\indicator}{I}

\newcommand{\x}{{\mathbf x}}

\newcommand{\abf}{{\mathbf a}}

\newcommand{\alphabf}{{\pmb \alpha}}

\newcommand{\Real}{{\mathbb R}}
\newcommand{\defeq}{\coloneqq}
\newcommand{\SigDO}{\mathbf M_0}
\newcommand{\I}{\mathbf I}
\newcommand{\pbinom}{{\tt pbinom}}
\newcommand{\A}{\mathbf A}

\newcommand{\E}{\mathbf E}
\newcommand{\W}{\mathbf W}
\newcommand{\Set}{\mathbf S}

\newcommand{\Hbf}{\mathbf H}

\newcommand{\Iden}{\mathbf I}
\newcommand{\LL}{\mathbf F}
\newcommand{\QQ}{\mathbf G}
\newcommand{\TT}{\mathbf t}
\newcommand{\sgn}{\textbf{sgn}}
\newcommand{\MU}[2]{\mu_{#2}(#1) }

\newcommand{\LOWER}[2]{\nu_{#2}(#1) }
\newcommand{\vertiii}[1]{{\left\vert\kern-0.25ex\left\vert\kern-0.25ex\left\vert #1 
    \right\vert\kern-0.25ex\right\vert\kern-0.25ex\right\vert}}
\newcommand{\vv}{\mathbf v}
\newcommand{\uu}{\mathbf u}
\newcommand{\z}{\mathbf z}
\newcommand{\y}{\mathbf y}
\newcommand{\ys}{\mathbf y_S}
\newcommand{\xibf}{\pmb \xi}
\newcommand{\HH}{H}
\newcommand{\EH}{h}
\newcommand{\PP}{\mathcal P}

\newcommand{\coh}{\mu}
\newcommand{\hs}[3]{\mathbb E \sqrt{L_{{#1}}(#2,#3)}}
\newcommand{\hb}[3]{\tau_{{#1}}(#2,#3)}

\newcommand{\intSet}[1]{\llbracket #1\rrbracket}

\begin{document}

\title{Local Identifiability of $l_1$-minimization Dictionary Learning: a Sufficient and Almost Necessary Condition}

\author{Siqi Wu%
  \thanks{Electronic address: \texttt{siqi@stat.berkeley.edu}}}
\affil{Department of Statistics, UC Berkeley}

\author{Bin Yu%
  \thanks{Electronic address: \texttt{binyu@berkeley.edu}}}
\affil{Department of Statistics and Department of EECS, UC Berkeley}



\maketitle

\begin{abstract}%
We study the theoretical properties of learning a dictionary from $N$ signals $\x_i\in \mathbb R^K$ for $i=1,...,N$ via $l_1$-minimization. We assume that $\x_i$'s are $i.i.d.$ random linear combinations of the $K$ columns from a complete (i.e., square and invertible) reference dictionary $\D_0 \in \mathbb R^{K\times K}$. Here, the random linear coefficients are generated from either the $s$-sparse Gaussian model or the Bernoulli-Gaussian model. First, for the population case, we establish a sufficient and almost necessary condition for the reference dictionary $\D_0$ to be locally identifiable, i.e., a local minimum of the expected $l_1$-norm objective function. Our condition covers both sparse and dense cases of the random linear coefficients and significantly improves the sufficient condition by Gribonval and Schnass (2010). In addition, we show that for a complete $\mu$-coherent reference dictionary, i.e., a dictionary with absolute pairwise column inner-product at most $\coh\in[0,1)$, local identifiability holds even when the random linear coefficient vector has up to $O(\mu^{-2})$ nonzeros on average. Moreover, our local identifiability results also translate to the finite sample case with high probability provided that the number of signals $N$ scales as $O(K\log K)$. 
\end{abstract}

\section{Introduction}
\label{sec:intro}
Expressing signals as sparse linear combinations of a dictionary basis has enjoyed great success in applications ranging from image denoising to audio compression. Given a known dictionary matrix $\D\in\mathbb R^{d\times K}$ with $K$ columns or atoms, one popular method to recover sparse coefficients $\alphabf\in\mathbb R^K$ of the signal $\x\in\mathbb R^d$ is through solving the convex $l_1$-minimization problem: 
\[\text{ minimize } \|\alphabf\|_1 \text{ subject to } \x = \D\alphabf.\] 
This approach, known as {\it basis pursuit} \citep{Chen1998}, along with many of its variants, has been studied extensively in statistics and signal processing communities. See, e.g. \cite{Donoho2003, Fuchs2004, Candes2005}.

For certain data types such as natural image patches, predefined dictionaries like the wavelets \citep{Mallat2008} are usually available. However, when a less-known data type is encountered, a new dictionary has to be designed for effective representations. Dictionary learning, or sparse coding, learns adaptively a dictionary from a set of training signals such that they have sparse representations under this dictionary \citep{Olshausen1997}. One formulation of dictionary learning involves solving a non-convex $l_1$-minimization problem  \citep{Plumbley2007, GS2010, Geng2011}. Concretely, 
define
\begin{align}
\label{basisPursuit}
l(\x,\D) = \min_{\alphabf\in \mathbb R^K}\{\|\alphabf\|_1, \text{ subject to } \x = \D\alphabf \}.
\end{align} 
We learn a dictionary from the $N$ signals $\x_i \in \mathbb R^d$ for $i=1,...,N$ by solving:
\begin{align}
\label{noiseless1}
\min_{\D\in \Dset} \ L_N(\D) =  \min_{\D\in \Dset} \frac{1}{N} \sum_{i=1}^N l(\mathbf{x}_i,\mathbf{D}).
\end{align}
\noindent Here, $\Dset \subset \mathbb R^{d\times K}$ is a constraint set for candidate dictionaries. In many signal processing tasks, learning an adaptive dictionary via the optimization problem (\ref{noiseless1}) and its variants is empirically demonstrated to have superior performance over fixed standard dictionaries \citep{Elad2006,Peyre2009,Grosse2012}. For a review of dictionary learning algorithms and applications, see \cite{Elad2010,Rubinstein2010,Mairal2014}.

Despite the empirical success of many dictionary learning formulations, relatively little theory is available to explain why they work. One line of research treats the problem of {\it dictionary identifiability}: if the signals are generated using a dictionary $\D_0$ referred to as the {\it reference dictionary}, under what conditions can we recover $\D_0$ by solving the dictionary learning problem? Being able to identify the reference dictionary is important when we interpret the learned dictionary. Let $\alphabf_i \in \mathbb R^K$ for $i=1,...,N$ be some random vectors. A popular signal generation model assumes that a signal vector can be expressed as a linear combination of the columns of the reference dictionary: $\x_i \approx \D_0\alphabf_i$ \citep{GS2010,Geng2011,Jenatton2012}. In this paper, we will study the problem of {\it local identifiability} of $l_1$-minimization dictionary learning (\ref{noiseless1}) under this generating model.

\vspace{2mm}

\noindent\textbf{Local identifiability.} A reference dictionary $\D_0$ is said to be {\it locally identifiable} with respect to an objective function $L(\D)$ if $\D_0$ is one of the local minima of $L$. The pioneer work of \citet{GS2010} (referred to as GS henceforth) analyzed the $l_1$-minimization problem (\ref{noiseless1}) for noiseless signals 
($\x_i = \D_0\alphabf_i$)
and complete ($d = K$ and full rank) dictionaries. Under a sparse Bernoulli-Gaussian model for the linear coefficients $\alphabf_i$'s, they showed that for a sufficiently incoherent reference dictionary $\D_0$, $N = O(K\log K)$ samples can guarantee local identifiability with respect to $L_N(\D)$ in (\ref{noiseless1}) with high probability. Still in the noiseless setting, \cite{Geng2011} extended the analysis to over-complete ($d>K$) dictionaries. More recently under the noisy linear generative model ($\x_i = \D_0\alphabf_i + \text{noise}$) and over-complete dictionary setting, \cite{Jenatton2012} developed the theory of local identifiability for (\ref{noiseless1}) with $l(\x,\D)$ replaced by the LASSO objective function of \cite{Tibshirani1996}. 
Other related works on local identifiability include \cite{Schnass2014a} and \cite{Schnass2014b}, who gave respectively sufficient conditions for the local correctness of the K-SVD \citep{Aharon2006b} algorithm and a maximum response formulation of dictionary learning. 
\vspace{2mm}

\noindent\textbf{Contributions.} There has not been much work on necessary conditions for local dictionary identifiability. Numerical experiments demonstrate that there seems to be a phase boundary for local identifiability (Figure \ref{fig:errorK10}). The bound implied by the sufficient condition in GS falls well below the simulated phase boundary, suggesting that their result can be further improved. Thus, even though theoretical results for the more general scenarios are available, we adapt the noiseless signals and complete dictionary setting of GS in order to find better local identifiability conditions. We summarize our major contributions below:
\begin{itemize}
\item For the population case where $N = \infty$, we establish a sufficient and almost necessary condition for local identifiability under both the $s$-sparse Gaussian model and the Bernoulli-Gaussian model. For the Bernoulli-Gaussian model, the phase boundary implied by our condition significantly improves the GS bound and agrees well with the simulated phase boundary (Figure \ref{fig:errorK10}). 
\item  We provide lower and upper bounds to approximate the quantities involved in our sufficient and almost necessary condition, as it generally requires to solve a series of second-order cone programs to compute those quantities.
\item As a consequence, we show that a $\coh$-coherent reference dictionary -- a dictionary with absolute pairwise column inner-product at most $\coh\in[0,1)$ -- is locally identifiable for sparsity level, measured by the average number of nonzeros in the random linear coefficient vectors, up to the order $O(\coh^{-2})$. Moreover, if the sparsity level is greater than $O(\coh^{-2})$, the reference dictionary is generally not locally identifiable. In comparison, instead of imposing condition on the sparsity level, the sufficient condition by GS demands the number of dictionary atoms $K = O(\coh^{-2})$, which is a much more stringent requirement. For over-complete dictionaries, \cite{Geng2011} requires the sparsity level to be of the order $O(\coh^{-1})$. It should also be noted that \cite{Schnass2014b} established the bound $O(\coh^{-2})$ for {\it approximate} local identifiability under a new response maximization formulation of dictionary learning. Our result is the first in showing that $O(\coh^{-2})$ is achievable and optimal for {\it exact} local recovery under the $l_1$-minimization criterion.  
\item We also extend our identifiability results to the finite sample case.
We show that for a fixed sparsity level, we need $N = O(K\log K)$ $i.i.d$ signals to determine whether or not the reference dictionary can be identified locally.  This sample requirement is the same as GS's and is the best known sample requirement among all previous studies on local identifiability.
\end{itemize}

\begin{figure}[h!]
\begin{center}
$\begin{array}{ccc}
\includegraphics[height=.45\textwidth]{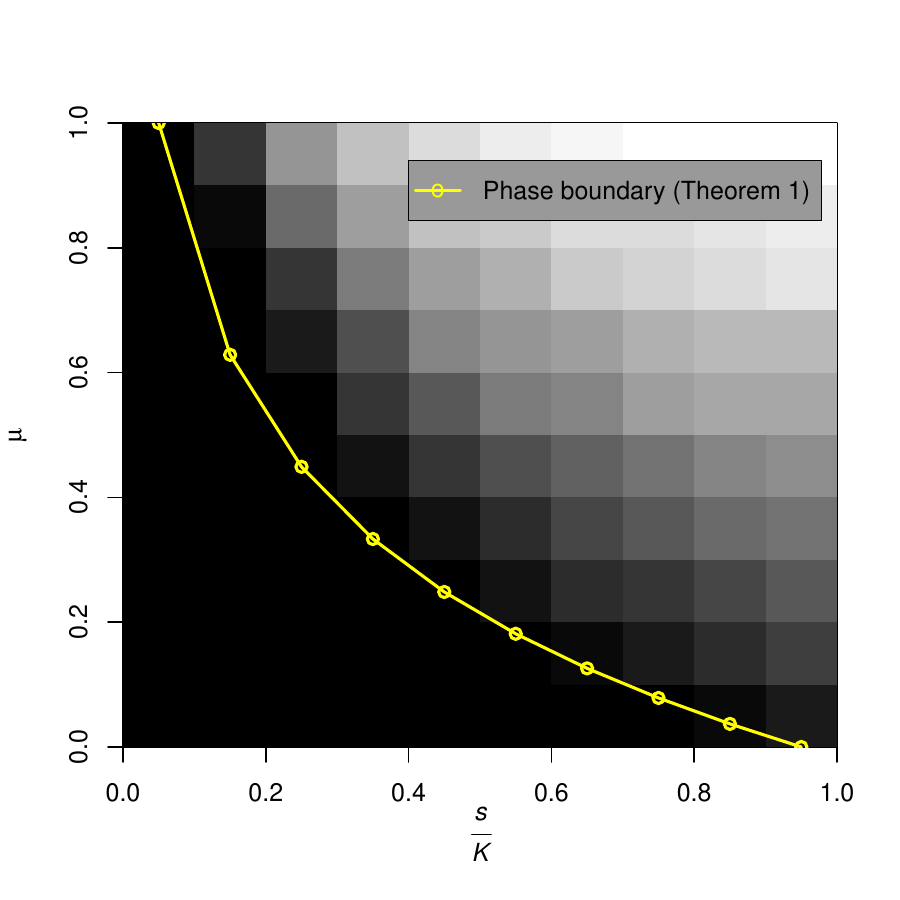} 
\includegraphics[height=0.45\textwidth]{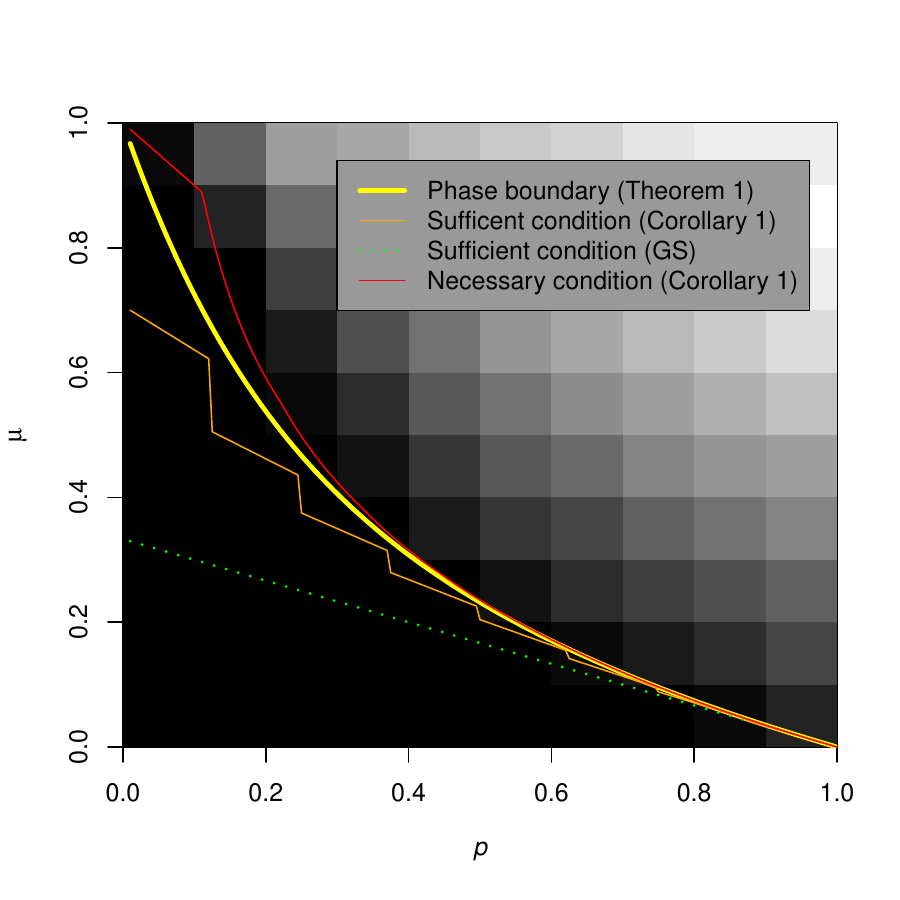}
\includegraphics[width=0.07\textwidth]{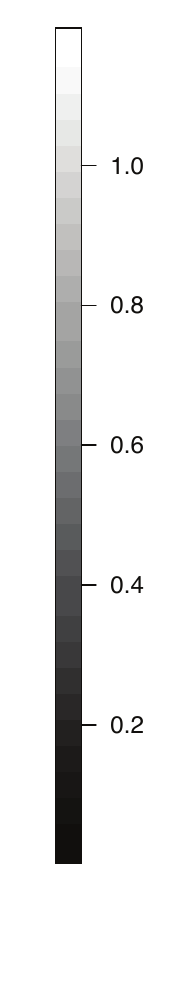}
\end{array}$
\end{center}
\caption{Local recovery error for the $s$-sparse Gaussian model (Left) and the Bernoulli($p$)-Gaussian model (Right). The parameter $s\in \{1,...,K\}$ is the number of nonzeros in each linear coefficient vector under the $s$-sparse Gaussian model, and $p\in (0,1]$ is the probability of an entry of the linear coefficient vector being nonzero under the Bernoulli($p$)-Gaussian model. The data are generated with the reference dictionary $\D_0\in\mathbb R^{10\times 10}$ (i.e. $K=10$)  satisfying $\D_0^T\D_0 = \mu\textbf{1}\textbf{1}^T+(1-\mu)\I$ for $\mu \in [0,1)$, see Example \ref{eg:allMu} for details. For each $(\mu,\frac{s}{K})$ or $(\mu,p)$ tuple, ten batches of $N=2000$ signals $\{\x_i\}_{i=1}^{2000}$ are generated according to the noiseless linear model $\x_i = \D_0\alphabf_i$, with $\{\alphabf_i\}_{i=1}^{2000}$ drawn $i.i.d$ from the $s$-sparse Gaussian model or $i.i.d$ from the Bernoulli($p$)-Gaussian model. For each batch, the dictionary is estimated through an alternating minimization algorithm in the \href{http://spams-devel.gforge.inria.fr}{{\tt SPAMS}} package  \citep{mairal2010}, with initial dictionary set to be $\D_0$. The grayscale intensity in the figure corresponds to the Frobenius error of the difference between the estimated dictionary and the reference dictionary $\D_0$, averaged for the ten batches. The ``phase boundary" curve corresponds to the theoretical boundary that separates the region of local identifiability (below the curve) and the region of local non-identifiability (above the curve) according to Theorem \ref{thm:popResult} of this paper. The ``Sufficient condition (Corollary 1)"  and ``Necessary condition (Corollary 1)" curves are the lower and upper bounds given by Corollary \ref{thm:popResult2} to approximate the exact phase boundary. Finally, the ``Sufficient condition (GS)" curve corresponds to the lower bound by GS. Note that for the $s$-sparse Gaussian model, the ``Sufficient condition (Corollary 1)"  and ``Necessary condition (Corollary 1)" curves coincide with the phase boundary.
}
\label{fig:errorK10}
\end{figure}
\vspace{2mm}

\noindent\textbf{Other related works.}  Apart from analyzing the local minima of dictionary learning, another line of research aims at designing provable algorithms for recovering the reference dictionary. \cite{Georgiev2005} and \cite {Aharon2006} proposed combinatorial algorithms and gave deterministic conditions for dictionary recovery which require sample size $N$ to be exponentially large in the number of dictionary atoms $K$. \cite{Spielman2012} established exact global recovery results for complete dictionaries through efficient convex programs. \cite{Agarwal2013} and \cite{Arora2014} proposed clustering-based methods to estimate the reference dictionary in the overcomplete setting. \cite{Agarwal2014a} and \cite{Arora2015} provided theoretical guarantees for their alternating minimization algorithms. \cite{Sun2015} proposed a non-convex optimization algorithm that provably recovers a complete reference dictionary for sparsity level up to $O(K)$. While in this paper we do not provide an algorithm, our identifiability conditions suggest theoretical limits of dictionary recovery for all algorithms attempting to solve the optimization problem (\ref{noiseless1}). In particular, in the regime where the reference dictionary is not identifiable, no algorithm can simultaneously solve (\ref{noiseless1}) and return the ground truth reference dictionary. 

Other related works include generalization bounds for signal reconstruction errors under the learned dictionary \citep{Maurer2010,Vainsencher2011,Mehta2012,Gribonval2013}, dictionary identifiability through combinatorial matrix theory \citep{Hillar2015}, as well as algorithms and theories for the closely related independent component analysis \citep{Comon1994,Arora2012ICA} and nonnegative matrix factorization \citep{Arora2012NMF,Recht2012}.


\vspace{2mm}

The rest of the paper is organized as follows: In Section \ref{sec:prelim}, we give basic assumptions and describe the two probabilistic models for signal generation. Section \ref{sec:pop} develops sufficient and almost necessary local identifiability conditions under both models for the population problem, and establishes lower and upper bounds to approximate the quantities involved in the conditions. In Section \ref{sec:finite}, we will present local identifiability results for the finite sample problem.
Detailed proofs for the theoretical results can be found in the Appendix.

\section{Preliminaries}
\label{sec:prelim}
\subsection{Notations}
For a positive integer $m$,  define $\intSet{m}$ to be the set of the first $m$ positive integers, $\{1,...,m\}$. The notation $\x[i]$ denotes the $i$-th entry of the vector $\x\in \mathbb R^m$. For a non-empty index set $S\subset \intSet{m}$, we denote by $|S|$ the set cardinality and $\x[S] \in \mathbb R^{|S|}$ the sub-vector indexed by $S$. We define $\x[-j] \defeq (\x[1],...,\x[j-1],\x[j+1],...,\x[m]) \in \mathbb R^{m-1}$ to be the vector $\x$ without its $j$-th entry. 

For a matrix $\A \in \mathbb R^{m\times n}$, we denote by $\A[i,j]$ its $(i,j)$-th entry. 
For non-empty sets $S \subset \intSet{m}$ and $T \subset \intSet{n}$, denote by $\A[S,T]$ the submatrix of $\A$ with the rows indexed by $S $ and columns indexed by $T$. Denote by $\A[i,]$ and $\A[,j]$ the $i$-th row and the $j$-th column of $\A$ respectively. Similar to the vector case, the notation $\A[-i,j]\in \mathbb R^{(m-1)\times n}$ denotes the $j$-th column of $\A$ without its $i$-th entry. 


For $p\geq 1$, the $l_p$-norm of a vector $\x \in \mathbb R^{m}$ is defined as $\|\x\|_p = (\sum_{i=1}^m |\x[i]|^p)^{1/p}$, with the convention that $\|\x\|_0 = |\{i:\x[i]\ne 0\}|$ and $\|\x\|_\infty = \max_i|\x[i]|$. For any norm $\|.\|$ on $\mathbb R^m$, the dual norm of $\|.\|$ is defined as $\|\mathbf{x}\|^* = \sup_{\mathbf{y}\ne 0}\frac{\mathbf{x}^T\mathbf{y}}{\|\mathbf{y}\|}$.  

For two sequences of real numbers $\{a_n\}_{n=1}^\infty$ and $\{b_n\}_{n=1}^\infty$, we denote by $a_n = O(b_n)$ if there is a constant $C>0$ such that $a_n\leq C b_n$ for all $n\geq 1$. For $a\in \mathbb R$, denote by $\lfloor a \rfloor$  the integer part of $a$ and $\lceil a \rceil$ the smallest integer greater than or equal to $a$. Throughout this paper, we shall agree that $\frac{0}{0}=0$.

\subsection{Basic assumptions}
We denote by $\Dset  \subset \mathbb R^{d\times K}$ the constraint set of dictionaries for the optimization problem (\ref{noiseless1}). In this paper,  since we focus on complete dictionaries, we assume $d=K$. As in GS, we choose $\Dset$ to be the {\it oblique manifold} \citep{Absil2008}: 
\[\Dset = \left\{\D\in\Real^{K\times K}: \ \|\D[,k]\|_2 = 1 \text{ for all $k = 1,...,K$}\right\}.\] 
We also call a column of the dictionary $\D[,k]$ an {\it atom} of the dictionary. Denote by $\D_0\in \Dset$ the {\it reference dictionary} -- the ground truth dictionary that generates the signals. With these notations, we now give a formal definition for local identifiability: 
\begin{definition} (Local identifiability) Let $L(\D):\Dset \to \mathbb R$ be an objective function.
  We say that the reference dictionary $\D_0$ is {\it locally identifiable} with respect to $L(\D)$ if $\D_0$ is a local minimum of $L(\D)$.
\end{definition}
\noindent\textbf{Sign-permutation ambiguity.} As noted by previous works GS and \cite{Geng2011}, there is an intrinsic sign-permutation ambiguity with the $l_1$-norm objective function $L(\D) = L_N(\D)$ of (\ref{noiseless1}).  Let $\mathbf \D' = \mathbf \D  \mathbf P \mathbf \Lambda$ for some permutation matrix $\mathbf P$ and diagonal matrix $\mathbf \Lambda$ with $\pm 1$ diagonal entries. It is easy to see that $\D'$ and $\D$ have the same objective value. Thus, the objective function $L_N(\D)$ has at least $2^n n!$ local minima. We can only recover $\D_0$ up to column permutation and column sign changes. 
\vspace{2mm}

Note that if the dictionary atoms are linearly dependent, the effective dimension is strictly less than $K$ and the problem essentially becomes over-complete. Since dealing with over-complete dictionaries is beyond the scope of this paper, we make the following assumption:

\vspace{2mm}
\noindent\textbf{Assumption I} {\it (Complete dictionaries). The reference dictionary $\D_0 \in \Dset \subset \mathbb R^{K\times K} $ is full rank.}
\vspace{2mm}

Let $\SigDO = \D_0^T\D_0$ be the {\it dictionary atom collinearity matrix} containing the inner-products between dictionary atoms. Since each dictionary atom has unit $l_2$-norm, $\SigDO[i,i] = 1$ for all $i\in\intSet{K}$. In addition, as $\D_0$ is full rank, $\SigDO$ is positive definite and $|\SigDO[i,j]|<1$ for all $i\ne j$. 


We assume that a signal is generated as a random linear combination of the dictionary atoms. In this paper, we consider the following two probabilistic models for the random linear coefficients:

\vspace{2mm}

\noindent\textbf{Probabilistic models for sparse coefficients.} Denote by $\z\in \mathbb R^m$ a random vector from the $K$-dimensional standard normal distribution.
\begin{description}
\item[\textbf{Model 1 -- $SG(s)$.}] Let $\Set$ be a size-$s$ subset uniformly drawn from all size-$s$ subsets of $\intSet{K}$. Define $\xibf \in \{0,1\}^K$ by setting $\xibf[j] = \indicator\{j\in \Set\}$ for $j\in \intSet{K}$, where $I\{.\}$ is the indicator function. Let $\alphabf\in\mathbb R^m$ be such that $\alphabf[j] = \xibf[j]\z[j]$. Then we say $\alphabf$ is drawn from the {\it $s$-sparse Gaussian model}, or $SG(s)$.
\item[\textbf{Model 2 -- $BG(p)$.}] For $j\in \intSet{K}$, let $\xibf[j]$'s be $i.i.d.$ Bernoulli random variable with success probability $p \in (0,1]$. Let $\alphabf\in\mathbb R^m$ be such that $\alphabf[j] = \xibf[j]\z[j]$. Then we say $\alphabf$ is drawn from the {\it Bernoulli($p$)-Gaussian model}, or $BG(p)$.
\end{description}
With the above two models we can formally state the following assumption for random signal generation:

\vspace{2mm}
\noindent\textbf{Assumption II} {\it (Signal generation). For $i \in \intSet{N}$, let $\alphabf_i$'s be either $i.i.d.$ $s$-sparse Gaussian vectors or $i.i.d.$ Bernoulli($p$)-Gaussian vectors. The signals $\x_i$'s are generated according to the noiseless linear model: }
\[\x_i = \D_0\alphabf_i.\]

\noindent\textbf{Remarks}: \\
\noindent(1) The above two models and their variants were studied in a number of prior theoretical works, including \cite{GS2010,Geng2011,Jenatton2012,Agarwal2014,Sun2015}. 

\noindent(2) By construction, a random vector generated from the $s$-sparse model has exactly $s$ nonzero entries. The data points $\x_i$'s therefore lie within the union of the linear spans of $s$ dictionary atoms (Figure \ref{fig:data} Left). The Bernoulli($p$)-Gaussian model, on the other hand, allows the random coefficient vector to have any number of nonzero entries ranging from $0$ to $K$ with a mean $pK$. As a result, the data points can be outside of the any sparse linear span of the dictionary atoms (Figure \ref{fig:data} Right). We refer readers to the remarks following Example \ref{eg:allMu} in Section \ref{sec:pop} for a discussion of the effect of non-sparse outliers on local identifiability. 

\noindent (3) Our local identifiability results can be extended to a wider class of sub-Gaussian distributions. However, such an extension will results in an increase complexity of the form of the quantities involved in our theorems. For proof of concept, for now we will only focus the standard Gaussian distribution. 
\vspace{2mm}

\begin{figure}[ht]
\begin{center}
$\begin{array}{cc}
\includegraphics[width=.5\textwidth]{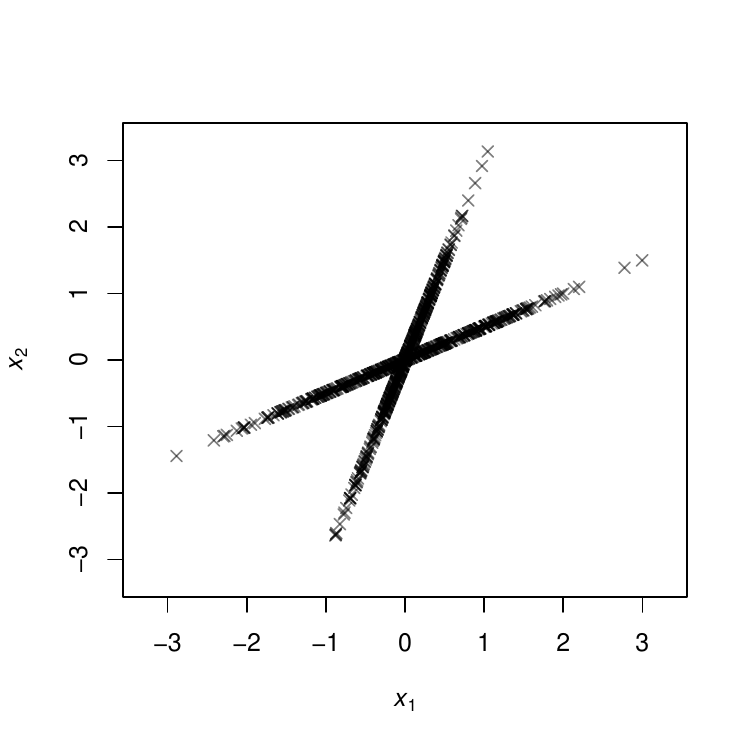} &
\includegraphics[width=0.5\textwidth]{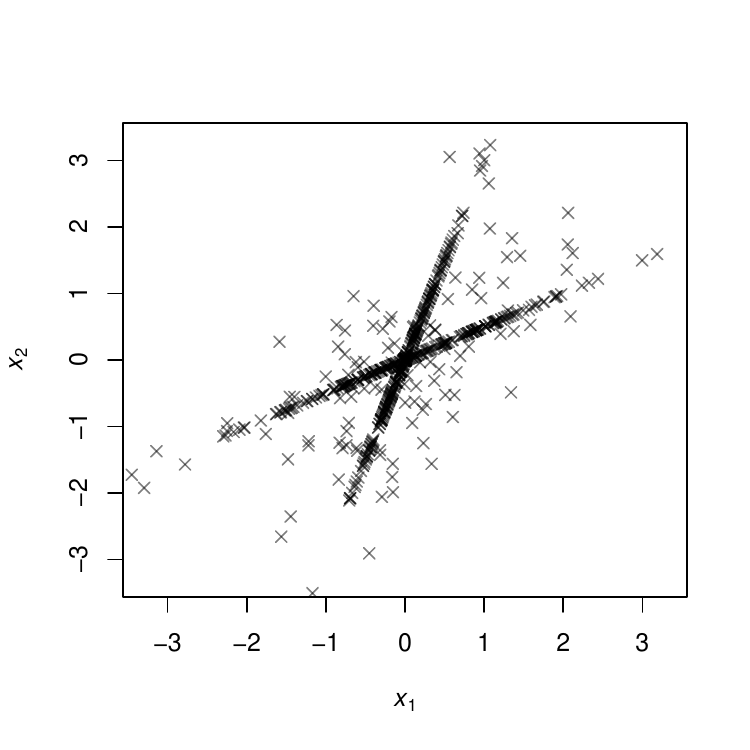}
\end{array}$
\end{center}
\caption{Data generation for $K=2$. Left: the $s$-sparse Gaussian model with $s=1$; Right: the Bernoulli($p$)-Gaussian model with $p=0.2$. The dictionary is constructed such that the inner product between the two dictionary atoms is $0.7$. A sample of $N=1000$ data points are generated for both models. For the $s$-sparse model all data points are perfectly aligned with the two lines corresponding to the two dictionary atoms. For the Bernoulli($p$)-Gaussian model, a number of data points fall outside the two lines. According to our Theorem \ref{thm:popResult} and \ref{thm:finiteBerGau}, despite those outliers and the high collinearity between the two atoms, the reference dictionary is still locally identifiable at the population level and with high probability for finite samples.}
\label{fig:data}
\end{figure}

In this paper, we study the problem of dictionary identifiability with respect to the population objective function $\mathbb E \ L_N(\D)$ (Section \ref{sec:pop}) and the finite sample objective function $L_N(\D)$ (Section \ref{sec:finite}).  In order to analyze these objective functions, it is convenient to define the following ``group LASSO"-type norms:



\begin{definition}
For an integer $m\geq2$ and $\w\in \mathbb R^{m}$. 
\begin{enumerate}
\item For $k\in\intSet{m}$, define
\[
\vertiii{\w}_k = \frac{\sum_{|S|=k}\|\w[S]\|_2}{{m-1 \choose k-1}}.
\]
\item For $p\in(0,1)$, define
\[
\vertiii{\w}_p = \sum_{k=0}^{m-1}{\tt pbinom}(k;m-1,p)\vertiii{\w}_{k+1}, 
\]
where {\tt pbinom} is the probability mass function of the binomial distribution:
\[{\tt pbinom}(k;n,p) = {n\choose k} p^k(1-p)^{n-k}.\]
\end{enumerate}
\end{definition}
\begin{remark}\hspace{1mm}\\
(1) Note that the above norms $\vertiii{\w}_k$ and $\vertiii{\w}_p$ are in fact the expected values of $|\w^T\alphabf|$ with the random vector $\alphabf$ drawn from the $SG(s)$ model and the $BG(p)$ model respectively. For invertible $\D\in\Dset$, it can be shown that the objective function for one signal $\x = \D_0\alphabf$ is 
\[l(\x,\D) = \|\Hbf \alphabf\|_1 = \sum_{j=1}^K |\Hbf[j,]\alphabf|,\] 
where $\Hbf = \D^{-1}\D_0$. Thus, taking the expectation of the objective function with respect to $\x$, we end up with a quantity involving either $\sum_{j=1}^K \vertiii{\Hbf[j,]}_s$ or $\sum_{j=1}^K \vertiii{\Hbf[j,]}_p$. This is the motivation of defining these norms. \\
\noindent (2) In particular, $\vertiii{\w}_1 = \|\w\|_1$ and $\vertiii{\w}_m = \|\w\|_2$.\\
\noindent (3) The norms defined above are special cases of the group LASSO penalty by \cite{Yuan2006}. For $\vertiii{\w}_k$, the summation covers all size-$k$ subsets of $\intSet{m}$. The normalization factor is the number of times $\w[i]$ appears in the numerator. Thus, $\vertiii{\w}_k$ is essentially the average of the $l_2$-norms of all size-$k$ sub-vectors of $\w$. On the other hand, $\vertiii{\w}_p$ is a weighted average of $\vertiii{\w}_k$'s with binomial probabilities. 
\end{remark}

\section{Population Analysis}
\label{sec:pop}
In this section, we establish local identifiability results for the case where infinitely many signals are observed. Denote by $\mathbb E \ l(\x_1,\D)$ the expectation of the objective function $l(\x_1,\D)$ of (\ref{basisPursuit}) with respect to the random signal $\x_1$. By the central limit theorem, as the number of signals $N$ tends to infinity, the empirical objective function $L_N(\D) = \frac{1}{N} \sum_{i=1}^N l(\x_i,\mathbf{D})$ converges almost surely to its population mean $\mathbb E \ l(\x_1,\D)$ for each fixed $\D\in \Dset$. Therefore the population version of the optimization problem (\ref{noiseless1}) is:
\begin{align}
\label{expNoiseless}
\min_{\D\in \Dset} & \ \mathbb E \ l(\x_1,\D) 
\end{align}
Note that we only need to work with $\D\in \Dset$ that is full rank. Indeed, if the linear span of the columns of $\D$ span$(\D)\neq \mathbb R^K$, then $\D_0\alphabf_1 \not \in \text{span}(\D)$ with nonzero probability. Thus $\D$ is infeasible with nonzero probability and so $\mathbb E \ l(\x_1,\D) = +\infty$. For a full rank dictionary $\D$, the following lemma gives the closed-form expressions for the expected objective function $\mathbb E \ l(\x_1,\D)$:
\begin{lemma} (Closed-form objective functions) Let $\D$ be a full rank dictionary in $\Dset$ and $\x_1 = \D_0\alphabf_1$ where $\alphabf_1\in \mathbb R^K$ is a random vector. For notational convenience, let $\Hbf = \D^{-1}\D_0$.
\label{lemma:closedForm}
\begin{enumerate}
\item If $\alphabf_1$ is generated according to the $SG(s)$ model with $s\in\intSet{K-1}$,  
\begin{align}
\label{g}
L_{SG(s)}(\D) \defeq \mathbb E \ l(\x_1,\D) = \sqrt{\frac{2}{\pi}}\frac{s}{K} \sum_{j=1}^K \vertiii{\Hbf[j,]}_s.
\end{align}
\item  If $\alphabf_1$ is generated according to the $BG(p)$ model with $p\in(0,1)$, 
\begin{align}
\label{f}
L_{BG(p)}(\D) \defeq \mathbb E \ l(\x_1,\D) = \sqrt{\frac{2}{\pi}}p\sum_{j=1}^K \vertiii{\Hbf[j,]}_p.
\end{align}
\end{enumerate}
For the non-sparse cases $s=K$ and $p=1$, we have
\[L_{SG(s)}(\D) = L_{BG(p)}(\D) = \sqrt{\frac{2}{\pi}}\sum_{j=1}^K \|\Hbf[j,]\|_2.\]
\end{lemma}
\begin{remark}
It can be seen from the above closed-form expressions that the two models are closely related. First of all, it is natural to identify $p$ with $\frac{s}{K}$, the fraction of expected number of nonzero entries in $\alphabf_1$. Next, by definition, $\vertiii{.}_p$ is a binomial average of $\vertiii{.}_k$. Therefore, the Bernoulli-Gaussian objective function $L_{BG(p)}(\D)$ can be treated as a binomial average of the $s$-sparse objective function $L_{SG(s)}(\D)$.
\end{remark}
\vspace{2mm}

By analyzing the above closed-form expressions of the $l_1$-norm objective function, we establish the following sufficient and almost necessary conditions for population local identifiability:
\begin{theoremM} (Population local identifiability) Recall that $\SigDO = \D_0^T\D_0$ and $\SigDO[-j,j]$ denotes the $j$-th column of $\SigDO$ without its $j$-th entry. Let $\vertiii{.}_s^*$ and $\vertiii{.}_p^*$ be the dual norm of $\vertiii{.}_s$ and $\vertiii{.}_p$ respectively.
\label{thm:popResult}  
\begin{enumerate}
\item ($SG(s)$ models) For $K\geq 2$ and $s \in \intSet{K-1}$, if
\[
\max_{j\in \intSet{K}}\vertiii{\SigDO[-j,j]}_s^* < 1 - \frac{s-1}{K-1}.
\]
then $\D_0$ is locally identifiable with respect to $L_{SG(s)}$.
\item ($BG(p)$ models) For $K\geq 2$ and $p \in (0,1)$, if
\[
\max_{j\in \intSet{K}}\vertiii{\SigDO[-j,j]}_p^* < 1 - p.
\]
then $\D_0$ is locally identifiable with respect to $L_{BG(p)}$.
\end{enumerate}
 Moreover, the above conditions are almost necessary in the sense that if the reverse strict inequalities hold, then $\D_0$ is not locally identifiable. 

On the other hand, if $s = K$ or $p = 1$, then $\D_0$ is not locally identifiable with respect to $L_{SG(s)}$ or $L_{BG(p)}$.
\end{theoremM}
\vspace{2mm}

\noindent\textbf{Proof sketch.} Let $\{\D_t\}_{t\in\mathbb R}$ be a collection of dictionaries $\D_t\in\Dset$ indexed by $t\in\mathbb R$ and $L(\D) = \mathbb E \ l(\x_1,\D)$ be the population objective function. The reference dictionary $\D_0$ is a local minimum of $L(\D)$ on the manifold $\Dset$ if and only if the following statement holds:  for any $\{\D_t\}_{t\in\mathbb R}$ that is a smooth function of $t$ with non-vanishing derivative at $t=0$, $L(\D_t)$ has a local minimum at $t = 0$.
For a fixed $\{\D_t\}_{t\in\mathbb R}$, to ensure that $L(\D_t)$ achieves a local minimum at $t=0$, it suffices to have the following one-sided derivative inequalities: 
\[\lim_{t\downarrow 0^+} \frac{L(\D_t) - L(\D_0)}{t} > 0 \text{ and } \lim_{t\uparrow 0^-} \frac{L(\D_t) - L(\D_0)}{t} < 0.\]
With some algebra, the two inequalities can be translated into the following statement:
\[\max_{j\in\intSet{K}}\Big|\SigDO[-j,j]^T\w\Big| < 
\left\{
\begin{array}{l l}
    1 - \frac{s-1}{K-1} & \quad \text{for $SG(s)$ models}\\
    1 - p & \quad \text{for $BG(p)$ models}
  \end{array} \right.\]
where $\w\in\mathbb R^{K-1}$ is a unit vector in terms of the norm $\vertiii{.}_s$ or $\vertiii{.}_p$ and it corresponds to the ``approaching direction" of $\D_t$ to $\D_0$ on $\Dset$ as $t$ tends to zero. Since $t=0$ has to be a local minimum for all smooth $\{\D_t\}_{t\in\mathbb R}$ or approaching directions, by taking the supremum over all such unit vectors the LHS of the above inequality becomes the dual norm of $\vertiii{.}_s$ or $\vertiii{.}_p$. On the other hand, $\D_0$ is not a local minimum if $\lim_{t\downarrow 0^+} (L(\D_t) - L(\D_0))/t< 0$ or $\lim_{t\uparrow 0^-}(L(\D_t) - L(\D_0))/t > 0$ for some $\{\D_t\}_{t\in\mathbb R}$. Thus our condition is also almost necessary. We refer readers to Section \ref{proof:popResult} for the detailed proof.
\vspace{2mm}

\noindent \textbf{Local identifiability phase boundary.} The conditions in Theorem \ref{thm:popResult} indicate that population local identifiability undergoes a phase transition. The following equations 
\[ \max_{j\in \intSet{K}}\vertiii{\SigDO[-j,j]}_s^* = 1 - \frac{s-1}{K-1} \text{ and }
\max_{j\in \intSet{K}}\vertiii{\SigDO[-j,j]}_p^* = 1 - p\]
define the local identifiability phase boundaries which separate the region of local identifiability, in terms of dictionary atom collinearity matrix $\SigDO$ and the sparsity level $s$ or $p$, and the region of local non-identifiability, under respective models.
\vspace{2mm}

\noindent \textbf{The roles of dictionary atom collinearity and sparsity.} Both the dictionary atom collinearity matrix $\SigDO$ and the sparsity parameter $s$ or $p$ play roles in determining local identifiability.  Loosely speaking, for $\D_0$ to be locally identifiable, neither can the atoms of $\D_0$ be too linearly dependent, nor can the random coefficient vectors that generate the data be too dense. For the $s$-sparse Gaussian model, the quantity $\max_{j\in \intSet{K}}\vertiii{\SigDO[-j,j]}_s^*$ measures the size of the off-diagonal entries of $\SigDO$ and hence the collinearity of the dictionary atoms. 
In addition, that quantity also depends on the sparsity parameter $s$. By Lemma \ref{lemma:normIncreasing} in the Appendix, $\max_{j \in \intSet{K}}\vertiii{\SigDO[-j,j]}_s^*$ is strictly increasing with respect to $s$ for $\SigDO$ whose upper-triangle portion contains at least two nonzero entries (if the upper-triangle portion contains at most one nonzero entry, then the quantity does not depend on $s$, see Example \ref{eg:oneMu}). Similar conclusion holds for the Bernoulli-Gaussian model. Therefore, the sparser the linear coefficients, the less restrictive the requirement on dictionary atom collinearity. 

On the other hand, for a fixed $\SigDO$, by the monotonicity of $\max_{j\in \intSet{K}}\vertiii{\SigDO[-j,j]}_s^*$ with respect to $s$, the collection of $s$ that leads to local identifiability is of the form $s<s^*(\SigDO)$ for some function $s^*$ of $\SigDO$.  Similarly for the Bernoulli-Gaussian model,  $p<p^*(\SigDO)$ for some function $p^*$ of $\SigDO$. 
\vspace{2mm}



Next, we will study some examples to gain more intuition for the local identifiability conditions.

\begin{example} ($1$-sparse Gaussian model)
\label{eg:oneSparse}
A full rank $\D_0$ is always locally identifiable at the population level under a $1$-sparse Gaussian model. Indeed, by Corollary \ref{lemma:normIncreasing} in the Appendix, $\vertiii{\SigDO[-j,j]}_1^* = \max_{i\ne j} |\SigDO[i,j]|<1$ for all $j\in\intSet{K}$. Thus, a full rank dictionary $\D_0$ always satisfies the sufficient condition.
\end{example}


\begin{example} ($(K-1)$-sparse Gaussian model) For $j \in\intSet {K}$, $\SigDO[-j,j]\in\mathbb R^{K-1}$. Thus by Lemma \ref{lemma:normIncreasing}, 
\[\vertiii{\SigDO[-j,j]}_{K-1}^* = \|\SigDO[-j,j]\|_2.\] 
Therefore the phase boundary under the $(K-1)$-sparse model is
\[
\max_{j\in\intSet{K}}\|\SigDO[-j,j]\|_2 =  \frac{1}{K}.
\]
\end{example}

%

\begin{example}
\label{eg:orthogonal}
(Orthogonal dictionaries) If $\SigDO = \I$, then 
\[\max_{j \in \intSet{K}}\vertiii{\SigDO[-j,j]}_s^*=\max_{j \in \intSet{K}}\vertiii{\SigDO[-j,j]}_p^*=0.\] 
Therefore orthogonal dictionaries are always locally identifiable if $s<K$ or $p<1$. 
\end{example}

\begin{example} (Minimally dependent dictionary atoms) 
\label{eg:oneMu}
Let $\mu\in(-1,1)$. Consider a dictionary atom collinearity matrix $\SigDO$ such that $\SigDO[1,2] = \SigDO[2,1] = \mu$ and $\SigDO[i,j] = 0$ for all other $i\ne j$. By Corollary \ref{thm:oneMu} in the Appendix, 
\[\max_{j \in \intSet{K}}\vertiii{\SigDO[-j,j]}_s^* = \max_{j \in \intSet{K}}\vertiii{\SigDO[-j,j]}_p^* = |\mu|.\]
Thus the phase boundaries under respective models are:
\[|\mu| = 1 - \frac{s-1}{K-1} \text{ and } |\mu| = 1 - p.\]
Notice that when $K=2$ and for the Bernoulli-Gaussian model, the phase boundary agrees well with the empirical phase boundary in the simulation result by GS (Figure 3 of the GS paper). 
\end{example}

\begin{example} 
\label{eg:allMu}
(Constant inner-product dictionaries) Let $\SigDO = \mu\textbf{1}\textbf{1}^T+(1-\mu)\I$, i.e. $\D_0[,i]^T\D_0[,j] = \mu$ for $1\leq i <  j \leq K$. Note that $\SigDO$ is positive definite if and only if $\mu\in(-\frac{1}{K-1},1)$. By Corollary \ref{thm:constMu} in the Appendix, we have 
\[
\vertiii{\SigDO[-j,j]}_s^* = \sqrt{s}|\mu|.
\]
Thus for the $s$-sparse model, the phase boundary is
\[
\sqrt{s}|\mu| =  1 - \frac{s-1}{K-1}.
\]
Similarly for the Bernoulli($p$)-Gaussian model, we have
\[
\vertiii{\SigDO[-j,j]}_p^* = |\mu|p(K-1)\Big(\sum_{k=0}^{K-1}\pbinom(k,K-1,p)\sqrt{k}\Big)^{-1}.
\]
Thus the phase boundary is
\[
|\mu| =  \frac{1-p}{p(K-1)}\sum_{k=0}^{K-1}\pbinom(k,K-1,p)\sqrt{k}.
\]
Figures \ref{fig:diffK} shows the phase boundaries for different dictionary sizes under the two models. As $K$ increases, the phase boundary moves towards the lower left of the region. This observation indicates that recovering the reference dictionary locally becomes increasingly difficult for larger dictionary size.
\end{example}

\begin{figure}[ht]
\begin{center}
$\begin{array}{cc}
\includegraphics[width=.5\textwidth]{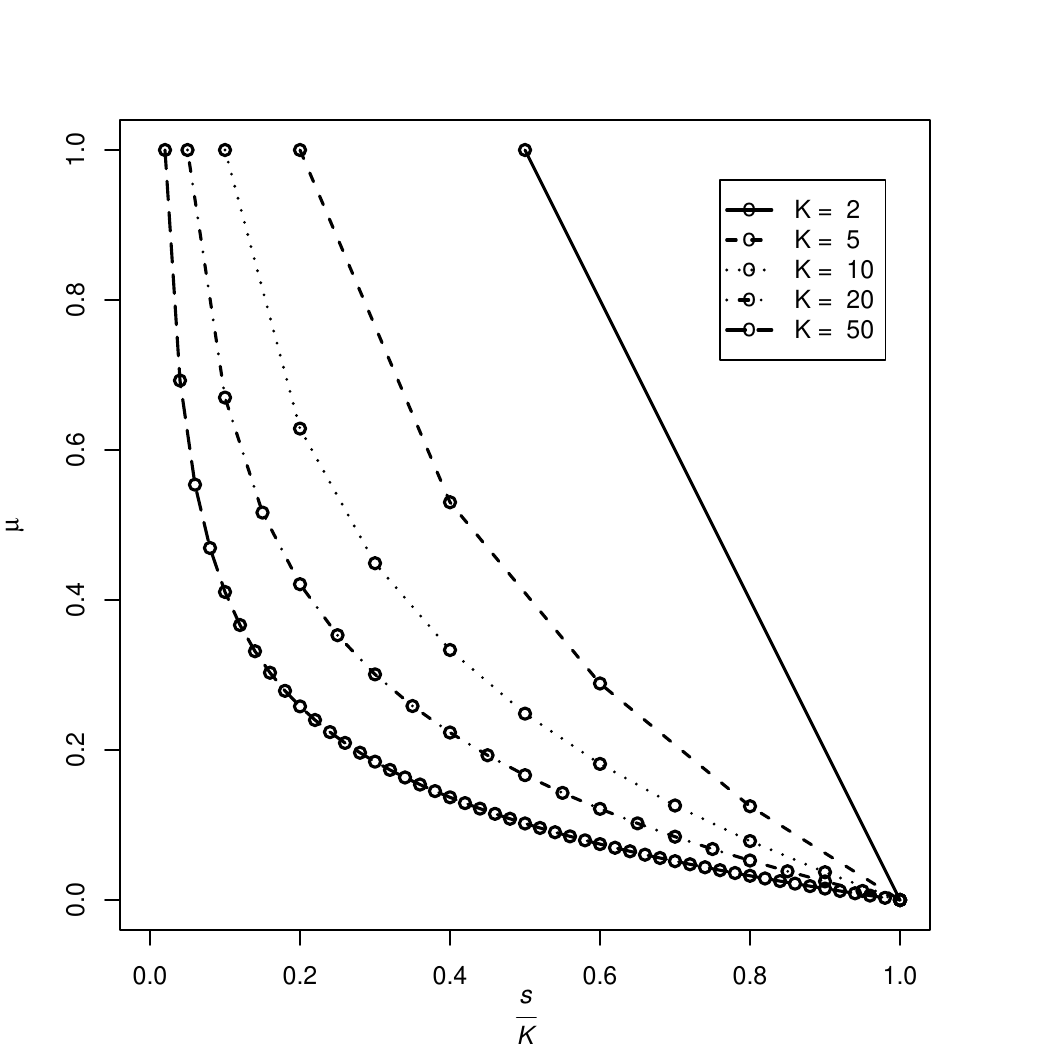} &
\includegraphics[width=0.5\textwidth]{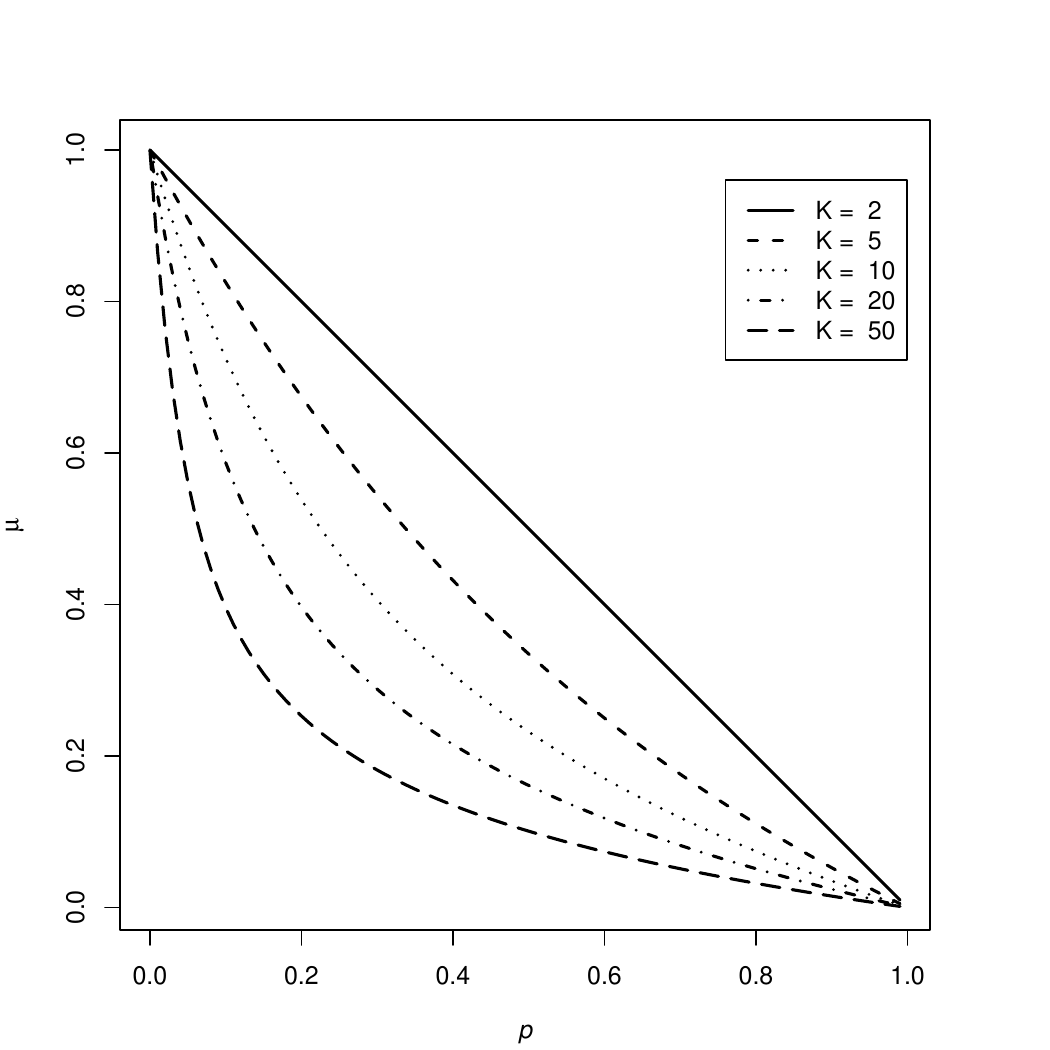}
\end{array}$
\end{center}
\caption{Local identifiability phase boundaries for constant inner-product dictionaries, under Left: the $s$-sparse Gaussian model; Right: the Bernoulli($p$)-Gaussian model. For each model, phase boundaries for different dictionary sizes $K$ are shown. Note that $\frac{s}{K}\in \{\frac{1}{K},\frac{2}{K}, \cdots ,1\}$ and $p \in (0,1]$. The area under the curves is the region where the reference dictionaries are locally identifiable at the population level. Due to symmetry, we only plot the portion of the phase boundaries for $\mu>0$. }
\label{fig:diffK}
\end{figure}

\noindent\textbf{The effect of non-sparse outliers.} Example \ref{eg:allMu} demonstrates how the presence of non-sparse outliers in the Bernoulli-Gaussian model (Figure \ref{fig:data} Right) affects the requirements for local identifiability. Set $p = \frac{s}{K}$ in order to have the same level of sparsity with the $SG(s)$ model. Applying Jensen's inequality, one can show that
\begin{align*}
\frac{1-p}{p(K-1)}\sum_{k=0}^{K-1}\pbinom(k,K-1,p)\sqrt{k}
< \frac{1}{\sqrt{s}}(1-\frac{s-1}{K-1}), 
\end{align*}
indicating that the phase boundary of the $s$-sparse models is always above that of the Bernoulli-Gaussian model with the same level of sparsity. The difference between the two phase boundaries is the extra price one has to pay, in terms of the collinearity parameter $\mu$, for recovering the dictionary locally in the presence of non-sparse outliers. One extreme example is the case where $s=1$ and correspondingly $p = \frac{1}{K}$. By Example \ref{eg:oneSparse}, under a $1$-sparse model the reference dictionary $\D_0$ is always locally identifiable if $|\mu|<1$. But for the $BG(\frac{1}{K})$ model, by the remark in Corollary \ref{thm:popResult2}, $\D_0$ is not locally identifiable if $|\mu| > 1 - \frac{1}{K}$. Hence, the requirement for $\mu$ in the presence of outliers is at least $\frac{1}{K}$ more stringent than that in the case of no outliers. 

However, such a difference diminishes as the number of dictionary atoms $K$ increases. Indeed, by Lemma \ref{lemma:normBounds}, one can show the following lower bound for the phase boundary of under the $BG(p)$ model:
\begin{align*}
\frac{1-p}{p(K-1)}\sum_{k=0}^{K-1}\pbinom(k,K-1,p)\sqrt{k}
\geq \frac{1-p}{\sqrt{p(K-1)+1}} \approx \frac{1}{\sqrt{s}}(1-\frac{s-1}{K-1}),
\end{align*}
for fixed sparsity level $p = \frac{s}{K}$ and large $K$. 
\vspace{2mm}

In general, the dual norms $\vertiii{.}_s^*$ and $\vertiii{.}_p^*$ have no closed-form expressions. According to Corollary \ref{col:SOCP} in the Appendix, computing those quantities involves solving a second order cone problem (SOCP) with a combinatoric number of constraints. The following Lemma \ref{lemma:normBounds}, on the other hand, gives computationally inexpensive approximation bounds.
\begin{definition}
\label{def:tau}
(Hyper-geometric distribution related quantities) Let $m$ be a positive integer and $d,k\in \{0\}\cup \intSet{m}$.
Denote by $L_m(d,k)$ the hypergeometric random variable with parameter $m$, $d$ and $k$, i.e. the number of 1's after drawing without replacement $k$ elements from $d$ 1's and $m-d$ 0's. 
Now for each $d\in\{0\}\cup \intSet{m}$, define the function $\hb{m}{d}{.}$ with domain on $[0,m]$ as follows: set $\hb{m}{d}{0}=0$. For $a\in (k-1,k]$ where $k\in\intSet{m}$, define 
\[\hb{m}{d}{a} =  \hs{m}{d}{k-1} + (\hs{m}{d}{k} - \hs{m}{d}{k-1})(a-(k-1)).\]
\end{definition}

\begin{lemma} 
\label{lemma:normBounds}
(Lower and upper bounds for $\vertiii{.}_s^*$ and $\vertiii{.}_p^*$) Let $m$ be a positive integer and $\z\in\mathbb R^m$.
\begin{enumerate}
\item For $s \in\intSet{m}$,
\[ \max \left(  \|\z\|_\infty, \sqrt{\frac{s}{m}}\max_{T\subset \intSet{m}}\frac{\|\z[T]\|_1}{\sqrt{|T|}} \right) \leq \frac{s}{m}\max_{T\subset \intSet{m}}\frac{\|\z[T]\|_1}{\hb{m}{|T|}{s}} \leq \vertiii{\z}_s^* \leq \max_{S\subset \intSet{m}, |S|=s}\|\z[S]\|_2.\]
\item For $p \in (0,1)$,
\[ 
\max\left( \|\z\|_\infty, \sqrt{p}\max_{T\subset \intSet{m}}\frac{\|\z[T]\|_1}{\sqrt{|T|}} \right) \leq p  \max_{T\subset \intSet{m}}\frac{\|\z[T]\|_1}{\hb{m}{|T|}{pm}}  \leq \vertiii{\z}_p^*\leq \max_{S\subset \intSet{m}, |S|=k}\|\z[S]\|_2. 
\]
where $k = \lceil p(m-1)+1 \rceil$.
\end{enumerate}
\end{lemma}
\begin{remark}\hspace{1mm} \\
\noindent(1) We refer readers to Lemma \ref{lemma:lower} and \ref{lemma:upper} for the detailed version of the above results. \\
\noindent(2) Since we agree that $\frac{0}{0}=0$, the case where $T=\emptyset$ does not affect taking the maximum of all subsets.\\
\noindent(3) Consider a sparse vector $\z = (z,0,...,0)^T\in \mathbb R^m$. By Corollary \ref{thm:oneMu},
\[\vertiii{\z}_s^* = \vertiii{\z}_p^* = |z| = \|\z\|_\infty = \max_{S\subset \intSet{m}, |S|=1}\|\z[S]\|_2.\] 
So the all the bounds are achievable by a sparse vector. \\
\noindent(4) Now consider a dense vector $\z = (z,...,z)^T\in\mathbb R^m$. By Corollary \ref{thm:constMu},
\[\vertiii{\z}_s^* = \sqrt{s}|z| = \sqrt{\frac{s}{m}}\max_{T\subset \intSet{m}}\frac{\|\z[T]\|_1}{\sqrt{|T|}} =\max_{S\subset \intSet{m}, |S|=s}\|\z[S]\|_2.\]
Thus the bounds for  $\vertiii{\z}_s^*$ can also be achieved by a dense vector. Similarly, by the upper-bound for $\vertiii{\z}_p^*$,
\[
\vertiii{\z}_p^* \leq \sqrt{pm+1}|z|.
\]
On the other hand,
\begin{align*}
\vertiii{\z}_p^* & \geq \sqrt{p}\max_{T\subset \intSet{m}}\frac{\|\z\|_1}{\sqrt{|T|}} 
= \sqrt{p}|z|\max_{T\subset \intSet{m}}\sqrt{|T|}
= \sqrt{pm}|z|.
\end{align*}
Thus both bounds for $\vertiii{\z}_p^*$ are basically the same for large $pm$. \\
\noindent (5) \textbf{Computation.} To compute the lower and upper bounds efficiently, we first sort the elements in $|\z|$ in descending order. Without loss of generality, we can assume that $|\z[1]|\geq |\z[2]| \geq ... \geq |\z[m]|$. Thus the upper-bound quantity becomes
\[
\max_{S\subset \intSet{m}, |S|=k}\|\z[S]\|_2 = (\sum_{i=1}^k \z[i]^2)^{1/2}.
\]  
For the lower-bound quantities, note that
\[
\max_{T\subset \intSet{m}}\frac{\|\z[T]\|_1}{\hb{m}{|T|}{k}}
= \max_{d\in\intSet{m}} \max_{T\subset \intSet{m},|T|=d} \frac{\sum_{i=1}^d|\z[i]|}{\hb{m}{d}{k}}
= \max_{d\in\intSet{m}} \frac{\sum_{i=1}^d|\z[i]|}{\hb{m}{d}{k}}.
\]
Thus, the major computation burden now is to compute $\hb{m}{d}{k} = \hs{m}{d}{k}$, for all $d\in \intSet{m}$. We do not know a closed-form formula for $\hs{m}{d}{k}$ except for $d=1$ or $d=m$. In practice, we compute $\hs{m}{d}{k}$ using its definition formula. On an OS X laptop with 1.8 GHz Intel Core i7 processor and 4GB of memory, the function {\tt dhyper} in the statistics software {\tt R} can compute $\hs{2000}{d}{1000}$ for all $d\in\intSet{2000}$ within 0.635 second. Note that the number of dictionary atoms in most applications is usually smaller than $2000$. 

In case $m$ is too large, the LHS lower bounds can be used. Note that 
\[
\max_{T\subset \intSet{m}}\frac{\|\z[T]\|_1}{\sqrt{|T|}}
=\max_{d\in\intSet{m}} \frac{\sum_{i=1}^d|\z[i]|}{\sqrt{d}},
\]
which can be computed easily. 
\end{remark}
\vspace{2mm}

For notational simplicity, we will define the following quantities that are involved in Lemma \ref{lemma:normBounds}:
\begin{definition}
For $a\in(0,K)$, define
\[
\LOWER{\SigDO}{a} = \max_{1\leq j\leq K} \max_{S\subset \intSet{K}, j\not\in S} \frac{\|\SigDO[S,j]\|_1}{\hb{K-1}{|S|}{a}}.
\]
%
\end{definition}

\begin{definition} (Cumulative coherence)
For $ k \in\intSet{K-1}$, define the {\it $k$-th cumulative coherence} of a reference dictionary $\D_0$ as
\[
\MU{\SigDO}{k} = \max_{1\leq j\leq K}\max_{S\subset \intSet{K}, |S|=k, j\not \in S}\|\SigDO[S,j]\|_2.
\]
\end{definition}
\begin{remark}
The above quantity is actually the $l_2$ analog of the $l_1$ $k$-th cumulative coherence defined in \cite{Jenatton2012}. Also, notice that $\MU{\SigDO}{1} = \max_{l\ne j}|\SigDO[l,j]|$ which is the plain mutual coherence of the reference dictionary.
 \end{remark}
 \vspace{2mm}
  
With the above definitions and as a direct consequence of the above Lemma \ref{lemma:normBounds}, we obtain a sufficient condition and a necessary condition for population local identifiability:
\begin{corollaryM}
\label{thm:popResult2}
Under the notations of Theorem \ref{thm:popResult}, 
we have
\begin{enumerate}[leftmargin=*]
\item Let $K\geq 2$ and $s\in\intSet{K-1}$. 
\begin{itemize}[leftmargin=*]
\item If $ \MU{\SigDO}{s}  < 1 - \frac{s-1}{K-1}$, then $\D_0$ is locally identifiable with respect to $L_{SG(s)}$; 
\item If $ \frac{s}{K-1}\LOWER{\SigDO}{s}> 1 - \frac{s-1}{K-1}$, then $\D_0$ is not locally identifiable with respect to $L_{SG(s)}$.
\end{itemize}
\item Let $K\geq 2$ and $p\in(0,1)$. 
\begin{itemize}[leftmargin=*]
\item If $ \MU{\SigDO}{k}  < 1 - p$, where $k = \lceil p(K-2)+1 \rceil$, then $\D_0$ is locally identifiable with respect to $L_{BG(p)}$; 
\item If $p\LOWER{\SigDO}{k} > 1 - p$, where $k = p(K-1)$, then $\D_0$ is not locally identifiable with respect to $L_{BG(p)}$.
\end{itemize}
\end{enumerate}
\end{corollaryM}
\begin{remark}\hspace{1mm}\\
\noindent (1) In particular, by Lemma \ref{lemma:normBounds}, if $\MU{\SigDO}{1} > 1- \frac{s-1}{K-1}$ or $\MU{\SigDO}{1}>1-p$, then $\D_0$ is not locally identifiable.\\
(2) We can also replace $ \frac{s}{K-1}\LOWER{\SigDO}{s}$ or $p\LOWER{\SigDO}{k}$ by the corresponding lower bound quantities in Lemma \ref{lemma:normBounds} which are easier to compute but give weaker necessary conditions.
\end{remark}
\vspace{2mm}

\noindent \textbf{Comparison with GS.} Corollary \ref{thm:popResult2} allows us to compare our local identifiability condition directly with that of GS. For the Bernoulli($p$)-Gaussian model, the population version of the sufficient condition for local identifiability by GS is: 
\begin{align}
\label{eqn:GSbound} 
\MU{\SigDO}{K-1} = \max_{1\leq j\leq K}\|\SigDO[-j,j]\|_2 < 1 - p.
\end{align}
Note that $\MU{\SigDO}{K-1} \geq \MU{\SigDO}{k}$ for $k \leq K -1$.
Thus, our local identifiability result implies that of GS. Moreover,  the quantity $\|\SigDO[-j,j]\|_2$ in inequality (\ref{eqn:GSbound}) computes the $l_2$-norm of the entire $\SigDO[-j,j]$ vector and is independent of the sparsity parameter $p$. On the other hand, in our sufficient condition $\max_{|S|=k, j\not \in S}\|\SigDO[S,j]\|_2$ computes the largest $l_2$-norm of all size-$k$ sub-vectors of $\SigDO[-j,j]$. Since $k = \lceil p(K-2)+1 \rceil$ is essentially $pK$, in the case where the model is sparse and the dictionary atoms collinearity matrix $\SigDO$ is dense,  the sufficient bound by GS is most conservative compared to ours. 

More concretely, let us consider constant inner-product dictionaries with parameter $\mu>0$ as in Example \ref{eg:allMu}.  The sufficient condition by GS and the sufficient condition given by Corollary \ref{thm:popResult2} are respectively
\[
\sqrt{K}\mu \leq 1 - p \text{ and } \sqrt{pK + 1}\mu \leq 1 - p,
\]
showing that the sufficient condition by GS is much more conservative for small value of $p$. See Figure \ref{fig:errorK10} for a graphical comparison of the bounds for $K=10$. 
\vspace{2mm}

\noindent\textbf{Local identifiability for sparsity level $O(\coh^{-2})$.} For notational convenience, let $\coh = \MU{\SigDO}{1}$ be the mutual coherence of the reference dictionary. For the $s$-sparse model, by Lemma \ref{lemma:normBounds}, $ \MU{\SigDO}{s} \leq \sqrt{s}\coh$. Thus the first part of the corollary implies a simpler sufficient condition:  
\[\sqrt{s}\coh < 1-\frac{s-1}{K-1}. \]
From the above inequality, it can be seen that if $1 - \frac{s-1}{K-1} > \delta$ for some $\delta > 0$, the reference dictionary is locally identifiable for sparsity level $s$ up to the order $ O(\coh^{-2})$. 

Similarly for the Bernoulli($p$)-Gaussian model, since 
\[ \MU{\SigDO}{k} \leq \sqrt{pK+1}\coh,\] 
we have the following sufficient condition for local identifiability:
\[\sqrt{pK+1}\coh\leq 1 - p.\]
As before, if $1-p>\delta$ for some $\delta>0$, the reference dictionary is locally identifiable for sparsity level $pK$ up to the order $O(\coh^{-2})$. On the other hand, the condition by GS  requires $K = O(\coh^{-2})$, which, does not take advantage of sparsity. 

In addition, by Example \ref{eg:allMu} and the remark under Lemma \ref{lemma:normBounds}, we also know that the sparsity requirement $O(\coh^{-2})$ cannot be improved in general. 

Our result seems to be the first to demonstrate  $O(\coh^{-2})$ is the optimal order of sparsity level for exact local recovery of a reference dictionary. For a predefined over-complete dictionary, classical results such as \cite{Donoho2003} and \cite{Fuchs2004} show that basis pursuit recovers an $s$-sparse linear coefficient vector with sparsity level $s$ up to the order $O(\coh^{-1})$. For over-complete dictionary learning, \cite{Geng2011} showed that exact local recovery is also possible for $s$-sparse model with $s$ up to $O(\coh^{-1})$. While our results are only for complete dictionaries, we conjecture that $O(\coh^{-2})$ is also the optimal order of sparsity level for over-complete dictionaries. In fact, \cite{Schnass2014b} proved that the response maximization criterion -- an alternative formulation of dictionary learning -- can approximately recover the over-complete reference dictionary locally with sparsity level $s$ up to $O(\coh^{-2})$. It will be of interest to investigate whether the same sparsity requirement hold for the $l_1$-minimization dictionary learning (\ref{noiseless1}) in the case of exact local recovery and over-complete dictionaries.



\section{Finite sample analysis}
\label{sec:finite}
In this section, we will present finite sample results for local dictionary identifiability. 
For notational convenience, we first define the following quantities:
\begin{align*}
\PP_1(\epsilon,N;\coh,K) & = 2\exp\left(-\frac{N\epsilon^2}{108K\coh}\right),\\
\PP_2(\epsilon,N;p,K) & = 2\exp\left(-p \frac{N\epsilon^2}{18p^2K+9\sqrt{2pK}}\right),\\
\PP_3(\epsilon,N;p,K) & =3\left(\frac{24}{\epsilon p}+1\right)^K\exp\left(-p\frac{N\epsilon^2}{360}\right).
\end{align*}

Recall that $\SigDO = \D_0^T\D_0$ and $\MU{\SigDO}{1}$ is the mutual coherence of the reference dictionary $\D_0$. The following two theorems give local identifiability conditions under the $s$-sparse Gaussian model and the Bernoulli-Gaussian model:
\begin{theoremM}
\label{thm:finiteSSparse}
(Finite sample local identifiability for $SG(s)$ models) Let $\alphabf_i\in \mathbb R^K$, $i\in\intSet{N}$, be $i.i.d$ $SG(s)$ random vectors with $s\in \intSet{K-1}$. The signals $\x_i$'s are generated as $\x_i = \D_0\alphabf_i$.  Assume $0<\epsilon\leq \frac{1}{2}$,
\begin{enumerate}
\item If 
\[\max_{j \in \intSet{K}}\vertiii{\SigDO[-j,j]}_s^*\leq 1-\frac{s-1}{K-1}-\sqrt{\frac{\pi}{2}}\epsilon,\]
then $\D_0$ is locally identifiable with respect to $L_N(\D)$ with probability exceeding
\[
1 - K^2\left(\PP_1(\epsilon,N; \MU{\SigDO}{1},K) + \PP_2(\epsilon,N;\frac{s}{K},K) + \PP_3(\epsilon,N;\frac{s}{K},K)\right).
\]
\item If 
\[\max_{j \in \intSet{K}}\vertiii{\SigDO[-j,j]}_s^*\geq 1-\frac{s-1}{K-1}+\sqrt{\frac{\pi}{2}}\epsilon,\] 
then $\D_0$ is not locally identifiable with respect to $L_N(\D)$ with probability exceeding
\[
1 - K\left(\PP_1(\epsilon,N; \MU{\SigDO}{1},K) + \PP_2(\epsilon,N;\frac{s}{K},K) + \PP_3(\epsilon,N;\frac{s}{K},K)\right).
\]
\end{enumerate}
\end{theoremM}

\begin{theoremM}
\label{thm:finiteBerGau}
(Finite sample local identifiability for $BG(p)$ models) Let $\alphabf_i\in \mathbb R^K$, $i\in\intSet{N}$, be $i.i.d$ $BG(p)$ random vectors with $p\in (0,1)$. The signals $\x_i$'s are generated as $\x_i = \D_0\alphabf_i$. Let $K_p = K+2p^{-1}$ and assume $0<\epsilon\leq \frac{1}{2}$,
\begin{enumerate}
\item If 
\[\max_{j \in \intSet{K}}\vertiii{\SigDO[-j,j]}_p^*\leq 1-p-\sqrt{\frac{\pi}{2}}\epsilon,\] 
then $\D_0$ is locally identifiable with respect to $L_N(\D)$ with probability exceeding
\[
1 - K^2\left(\PP_1(\epsilon,N; \MU{\SigDO}{1},K_p) + \PP_2(\epsilon,N;p,K_p) + \PP_3(\epsilon,N;p,K)\right).
\]
\item If 
\[\max_{j \in \intSet{K}}\vertiii{\SigDO[-j,j]}_p^*\geq 1-p+\sqrt{\frac{\pi}{2}}\epsilon,\]
then $\D_0$ is not locally identifiable with respect to $L_N(\D)$ with probability exceeding
\[
1 - K\left(\PP_1(\epsilon,N; \MU{\SigDO}{1},K_p) + \PP_2(\epsilon,N;p,K_p) + \PP_3(\epsilon,N;p,K)\right).
\]
\end{enumerate}
\end{theoremM}
The conditions for finite sample local identifiability are essentially identical as their population counterparts. The only difference is an margin of $\sqrt{\frac{\pi}{2}}\epsilon$ on the RHS of the inequalities. Such a margin appears in the conditions because of our proof techniques: we show that the derivative of $L_N$ is within $O(\epsilon)$ of its expectation and then impose conditions on the expectation. 
\vspace{2mm}

\noindent\textbf{Sample size requirement.} The theorems indicate that if the number of signals is a multiple of the following quantity,
\[
\text{For $SG(s)$: }\frac{1}{\epsilon^2}\max\left\{
\MU{\SigDO}{1} K\log K, \  s\log K, 
\ \frac{K}{s}K\log \left(\frac{K}{\epsilon s}\right)
\right\} 
\]
\[
\text{For $BG(p)$: }\frac{1}{\epsilon^2}\max\left\{
\MU{\SigDO}{1} K\log K, \  pK\log K, 
\ \frac{1}{p}K\log \left(\frac{1}{\epsilon p}\right)
\right\} 
\]
 then with high probability we can determine whether or not $\D_0$ is locally identifiable. For the ease of analysis, let us now treat $\epsilon$ as a constant. Thus, in the worst case, the sample size requirements for the two models are, respectively,
\[O(\frac{K\log K}{\frac{s}{K}}) \text{ and } O(\frac{K\log K}{p}).\] 
Apart from playing a role in determining whether $\D_0$ is locally identifiable, the sparsity parameters $s$ and $p$ also affect the sample size requirement. As discussed in the population results, the sparser the linear coefficient $\alphabf_i$, the less constraint on the dictionary atom collinearity. However, with finite samples, more signals are needed to guarantee the validity of the local identifiability conditions for sparse models. 

Our sample size requirement is similar to that of GS, who shows that $O(\frac{K\log K}{p (1-p)})$ signals is enough for locally recovering an incoherent reference dictionary. Our result indicates the $1-p$ factor in the denominator can be removed. 
\vspace{2mm}
 
The following two corollaries are the finite sample counterparts of Corollary \ref{thm:popResult2}. 
\begin{corollaryM}
\label{thm:finiteSSparse2}
Under the same assumptions of Theorem \ref{thm:finiteSSparse},
\begin{enumerate}
\item (Sufficient condition for $SG(s)$ models) If
\[
\MU{\SigDO}{s}\leq 1-\frac{s-1}{K-1}-\sqrt{\frac{\pi}{2}}\epsilon,
\] 
then $\D_0$ is locally identifiable with respect to $L_N(\D)$, with the same probability bound in the first part of Theorem \ref{thm:finiteSSparse}.
\item (Necessary condition for $SG(s)$ models) If
\[
\frac{s}{K-1}\LOWER{\SigDO}{s} \geq 1-\frac{s-1}{K-1}+\sqrt{\frac{\pi}{2}}\epsilon,
\] 
then $\D_0$ is not locally identifiable with respect to $L_N(\D)$, with the same probability bound in the second part of Theorem \ref{thm:finiteSSparse}.
\end{enumerate}
\end{corollaryM}

\begin{corollaryM}
\label{thm:finiteBerGau2}
Under the same assumptions of Theorem \ref{thm:finiteBerGau},
\begin{enumerate}
\item (Sufficient condition for $BG(p)$ models) Let $k = \lceil p(K-1)+1 \rceil$. If 
\[
\MU{\SigDO}{k} \leq 1-p-\sqrt{\frac{\pi}{2}}\epsilon,
\] 
then $\D_0$ is locally identifiable with respect to $L_N(\D)$, with the same probability bound in the first part of Theorem \ref{thm:finiteBerGau}.
\item  (Necessary condition for $BG(p)$ models) Let $k = p(K-1)$. If 
\[
p\LOWER{\SigDO}{k} \geq 1-p+\sqrt{\frac{\pi}{2}}\epsilon,
\] 
then $\D_0$ is not locally identifiable with respect to $L_N(\D)$, with the same probability bound in the second part of Theorem \ref{thm:finiteBerGau}.
\end{enumerate}
\end{corollaryM}
\begin{remark}
As before, denote by $\mu \in [0,1)$ the coherence of the reference dictionary. The above two corollaries together with the remark under Corollary \ref{thm:popResult2} indicate that the reference dictionary is locally identifiable with high probability for sparsity level $s$ or $pK$ up to the order $O(\coh^{-2})$.
\end{remark}
\vspace{2mm}

\noindent\textbf{Proof sketch for Theorem \ref{thm:finiteSSparse} and \ref{thm:finiteBerGau}.} 
Similar to the population case, by taking one-sided derivatives of $L_N(\D_t)$ with respect to $t$ at $t=0$ for all smooth $\{\D_t\}_{t\in \mathbb R}$, we derive a sufficient and almost necessary algebraic condition for the reference dictionary $\D_0$ to be a local minimum of $L_N(\D)$. Using the concentration inequalities in Lemma 
\ref{lemma:L} - \ref{lemma:T}, we show that the random quantities involved in the algebraic condition are close to their expectations with high probability. The population results for local identifiability can then be applied. The proofs for the two signal generation models are conceptually the same after establishing Lemma \ref{lemma:normRelation} to relate the $\vertiii{.}_p^*$ norm to the $\vertiii{.}_s^*$ norm. The detailed proof can be found in Section \ref{proof:finiteResult}.
\vspace{2mm}

\noindent\textbf{Comparison with the proof by GS.} The key difference between our analysis and that of GS is that we use an alternative but equivalent formulation of dictionary learning. Instead of (\ref{noiseless1}), GS studied the following problem: 
\begin{align}
\label{noiseless0}
\min_{\D\in\Dset,\alphabf_i} & \frac{1}{N} \sum_{i=1}^N \|\alphabf_i\|_1 \\
 \text{subject to } & \x_i = \D\alphabf_i \text{ for all $i\in\intSet{N}$.} \nonumber
\end{align}
Note that the above formulation optimizes jointly over $\D$ and $\alphabf_i$ for $i\in\intSet{N}$, as opposed to optimizing with respect to the only parameter $\D$ in our case. For complete dictionaries, this formulation is equivalent to the formulation in $(\ref{noiseless1})$ in the sense that $\hat{\D}$ is a local minimum of (\ref{noiseless1}) if and only if $(\hat{\D},\hat{\D}^{-1}[\x_1,...,\x_N])$ is a local minimum of (\ref{noiseless0}), see Remark 3.1 of GS. The number of parameters to be estimated in (\ref{noiseless0}) is $(K-1)K+KN$, compared to $(K-1)K$ free parameters in $(\ref{noiseless1})$. The growing number of parameters make the formulation employed by GS less trackable to analyze under a signal generation model.

GS did not study the population case. In their analysis, GS first obtained an algebraic condition for local identifiability that is sufficient and almost necessary. However, their condition is convoluted due to its direct dependence on the signals $\x_i$'s. In order to make their condition more explicit in terms of  dictionary atom collinearity and sparsity level, they then investigated the condition under the Bernoulli-Gaussian model. During the probabilistic analysis, the sharp algebraic condition was weakened, resulting in a sufficient condition that is far from being necessary. 

In contrast, we start with probabilistic generative models. The number of parameters is not growing as $N$ increases, which, allows us to study the population problem directly and to apply concentration inequalities for the finite sample problem. There is little loss of information during the process of obtaining identifiability results from first principles. Therefore, studying the optimization problem (\ref{noiseless1}) instead of (\ref{noiseless0}) is the key to establishing an interpretable sufficient and almost necessary local identifiability condition.

\section{Conclusions and future work}
We have established sufficient and almost necessary conditions for local dictionary identifiability under both the $s$-sparse model and the Bernoulli-Gaussian model in the case of noiseless signals and complete dictionaries. For finite sample with a fixed sparsity level, we have shown that as long as the number of $i.i.d$ signals scales as $O(K\log K)$, with high probability we can determine whether or not a reference dictionary is locally identifiable by checking the identifiability conditions. 

There are several directions for future research. First of all, here we only study the local behaviors of the $l_1$-norm objective function.
As pointed out by GS, numerical experiments in two dimensions suggest that local minima are in fact global minima, see Figure 2 of GS. Thus, it is of interest to investigate whether the conditions developed in this paper for local identifiability are also sufficient and almost necessary for global identifiability. 

Moreover, one can extend our results to a wider class of sub-Gaussian distributions other than the standard Gaussian distribution considered in this paper. We foresee little technical difficulties for this extension. However, it should be noted that the quantities involved in our local identifiability conditions, i.e. the $\vertiii{.}_s^*$ and $\vertiii{.}_p^*$ norms, are consequences of the standard Gaussian assumption. Under a different distribution, it can be even more challenging to compute and approximate those quantities. 

Finally, it would be also desirable to improve the sufficient condition by \cite{Geng2011} and \cite{Jenatton2012}  for over-complete dictionaries and noisy signals. One of the implications of our identifiability condition is that local recovery is possible for sparsity level up to the order $O(\coh^{-2})$ for a $\mu$-coherent reference dictionary. We conjecture the same sparsity requirement holds for the over-complete and/or the noisy signal case. In either case, the closed-form expression for the objective function is no longer available. A full characterization of local dictionary identifiability requires us to develop new techniques to analyze the local behaviors of the objective function.


\section*{Acknowledgments}
This research is supported in part by the Citadel Fellowship at the Department of Statistics of UCB, NHGRI grant U01HG007031,  NSF grants DMS-1107000,  CDS\&E-MSS 1228246, DMS-1160319 (FRG), ARO grant W911NF-11-1-0114, AFOSR grant FA9550-14-1-0016,
and the Center for Science of Information (CSoI), a US NSF Science and Technology Center, under grant agreement CCF-0939370. The authors would like to thank Sivaraman Balakrishnan and Yu Wang for helpful comments on the manuscript.

\appendix
\section{Appendix}
\label{sec:appendix}
Let $\Obj(\D)$ be a function of $\D\in \Dset$ and $\{\D_t\}_{t\in\mathbb R}$ be the collection of dictionaries $\D_t\in\Dset$ parameterized by $t\in\mathbb R$.  By definition, $\{\D_t\}_{t\in\mathbb R}$ passes through the reference dictionary $\D_0$ at $t=0$. To ensure that $\D_0$ is a local minimum of $\Obj(\D)$, it suffices to have
\[
\lim_{t\downarrow 0^+}\frac{\Obj(\D_t) - \Obj(\D_0)}{t}>0 \text{ and } 
\lim_{t\uparrow 0^-}\frac{\Obj(\D_t) - \Obj(\D_0)}{t}<0,
\]
for all $\{\D_t\}_{t\in\mathbb R}$ that is a smooth function of $t$. On the other hand, if either of the above strict inequalities holds in the reverse direction for some smooth $\{\D_t\}_{t\in\mathbb R}$,  then $\D_0$ is not a local minimum of $\Obj(\D)$.

Since $\D_0$ is full rank by assumption, the minimum eigenvalue of $\SigDO = \D_0^T\D_0$ is  strictly greater than zero. By continuity of the minimum eigenvalue of $\D_t^T\D_t$ (see e.g. Bauer-Fike Thoerem),  when $\D_t$ and $\D_0$ are sufficiently close, $\D_t$ should also be full rank. Thus without loss of generality we only need to work with full rank dictionary $\D_t$. For any full rank $\D\in \Dset$, there is a full rank matrix $\A\in \mathbb R^{K\times K}$ such that $\D = \D_0 \A$. For any $k\in\intSet{K}$, by the constraint $\|\D[,k]\|_2 = 1$, the matrix $\A$ should satisfy $\A[,k]^T\SigDO \A[,k] = 1$. Define the set for all such $\A$'s as:
\begin{align}
\label{setA}
\mathcal A = \{\A \in \mathbb R^{K\times K}: \A \text{ is invertible and } \A[,k]^T\SigDO \A[,k] = 1 \text{ for all $k\in\intSet{K}$}\}.
\end{align}
It follows immediately that the set $\{\D_0\A:\A\in\mathcal A\}$ is the collection of $\D\in \Dset$ such that $\D$ is full rank. Thus, to ensure that $\D_0$ is a local minimum of $\Obj(\D)$, it suffices to show
\begin{align}
\label{eqn:gradientPlus}
\Delta^+(\Obj,\{\A_t\}_{t}) \defeq & \lim_{t\downarrow 0^+}\frac{\Obj(\D_0\A_t) - \Obj(\D_0)}{t}> 0, \\
\label{eqn:gradientMinus}
\Delta^-(\Obj,\{\A_t\}_{t}) \defeq & \lim_{t\uparrow 0^-}\frac{\Obj(\D_0\A_t) - \Obj(\D_0)}{t}< 0,
\end{align}
for all smooth functions $\{\A_t\}_{t\in\mathbb R}$ with $\A_t \in \mathcal A$ and $\A_0 = \Iden$. In addition, to demonstrate that $\D_0$ is not a local minimum of $\Obj(\D)$, it suffices to have (\ref{eqn:gradientPlus}) or (\ref{eqn:gradientMinus}) to hold in the reverse direction for some $\{\A_t\}_t$ with the aforementioned properties. 
We will be using this characterization of local minimum to prove local identifiability results for both the population case and the finite sample case.

\subsection{Proofs of the population results}
\subsubsection{Proof of Lemma \ref{lemma:closedForm}}
\label{proof:closedForm}
\begin{proof}
Since $\mathbb E ||\Hbf\alphabf_1||_1 = \sum_{j=1}^K  \mathbb E |\Hbf[j,]\alphabf_1|$, it suffices to compute $\mathbb E |\Hbf[j,]\alphabf_1|$. 
Let $S$ be any nonempty subset of $\intSet{K}$. Recall that the random variable $\Set_1\subset \intSet{K}$ denotes the support of random coefficient $\alphabf_1$. Conditioning on the event $\{\Set_1 = S\}$, the random variable $\Hbf[j,]\alphabf_1$ follows a normal distribution with mean $0$ and standard deviation $\|\Hbf[j,S]\|_2$. Hence
\begin{align*}
\mathbb E \left|\Hbf[j,]\alphabf_1\right| = \mathbb E [\mathbb E[\left|\Hbf[j,]\alphabf_1\right| |\Set_1]] = \sqrt{\frac{2}{\pi}}\mathbb E \|\Hbf[j,\Set_1]\|_2.
\end{align*}
\noindent (1) Under the $s$-sparse Gaussian model, $\mathbb P(\Set_1 = S) = {K\choose s}^{-1}$ for any $|S|=s$. Thus we have
\begin{align*}
\mathbb E \|\Hbf[j,\Set_1]\|_2 
 = {K \choose s}^{-1}\sum_{S:|S|=s} \|\Hbf[j,S]\|_2
=\frac{s}{K} \vertiii{\Hbf[j,]}_s.   
\end{align*}
Hence the objective function for the $s$-sparse Gaussian model is
\begin{align*}
L_{SG(s)}(\D)  = \sum_{j=1}^K \mathbb E \left|\Hbf[j,]\alphabf_1\right|  
= \sqrt{\frac{2}{\pi}}\frac{s}{K} \sum_{j=1}^K \vertiii{\Hbf[j,]}_s.
\end{align*}
In particular, for $s=K$, $\vertiii{\Hbf[j,]}_{K} = \|\Hbf[j,]\|_2$ and so
\[L_{SG(s)}(\D) = \sqrt{\frac{2}{\pi}}\sum_{j=1}^K \|\Hbf[j,]\|_2.\]
\noindent (2) Under the Bernoulli($p$)-Gaussian model, $\mathbb P(\Set_1 = S) = p^{|S|}(1-p)^{K-|S|}$. So we have
\begin{align*}
\mathbb E [\|\Hbf[j,\Set_1]\|_2]
&  =\sum_{k=1}^K\sum_{S:|S|=k} p^k(1-p)^{K-k} \|\Hbf[j,S]\|_2  \\
& = p  \sum_{k=0}^{K-1} \pbinom(k;K-1,p)  \vertiii{\Hbf[j,]}_{k+1}.
\end{align*}
Therefore for $p\in(0,1)$, the objective function under the Bernoulli-Gaussian model is
\begin{align*}
L_{BG(p)}(\D)  = \sum_{j=1}^K \mathbb E \left|\Hbf[j,]\alphabf_1\right|  
& = \sqrt{\frac{2}{\pi}}p\sum_{j=1}^K \vertiii{\Hbf[j,]}_p.
\end{align*}
Finally, if $p=1$, we have 
\[L_{BG(p)}(\D) = \sqrt{\frac{2}{\pi}}\sum_{j=1}^K \|\Hbf[j,]\|_2.\]
\end{proof}

\subsubsection{Proof of Theorem \ref{thm:popResult}}
\label{proof:popResult}
\begin{proof}
(1) Let us first consider the $s$-sparse Gaussian model. By (\ref{eqn:gradientPlus}) and (\ref{eqn:gradientMinus}), to ensure that $\D_0$ is a local minimum of $L_{SG(s)}(\D)$, it suffices to show
\begin{align}
\label{eqn:gradientPM}
\Delta^+(L_{SG(s)}, \{\A_t\}_t) > 0 \text{ and }
\Delta^-(L_{SG(s)}, \{\A_t\}_t) < 0,
\end{align}
for all smooth functions $\{\A_t\}_t$ with $\A_t \in \mathcal A$, where $\mathcal A$ is defined in (\ref{setA}), and $\A_0 = \Iden$. Note that by Lemma \ref{lemma:closedForm}, 
\begin{align}
\label{plusForm}
\Delta^+(L_{SG(s)}, \{\A_t\}_t) = \sqrt{\frac{2}{\pi}}\frac{s}{K}\sum_{j=1}^K\lim_{t\downarrow 0^+}\frac{1}{t}\left(\vertiii{\A^{-1}_t[j,]}_s - \vertiii{\Iden[j,]}_s\right).
\end{align}
For a fixed $j\in\intSet{K}$, we have
\begin{align}
\label{jsplit}
{K-1\choose s-1}\vertiii{\A^{-1}_t[j,]}_s = \sum_{S:|S|=s,j\in S}\|\A^{-1}_t[j,S]\|_2 + \sum_{S:|S|=s,j\not \in S}\|\A^{-1}_t[j,S]\|_2
\end{align}
Denote by $\dot{\A}_0\in\mathbb R^{K\times K}$ the derivative of $\{\A_t\}_t$ at $t=0$. Since $\A_t\in \mathcal A$ for all $t \in \mathbb R$, it can be shown that 
\begin{align}
\label{mu}
\SigDO[,k]^T\dot{\A}_0[,k] = 0 \ \text{ for all $k\in\intSet{K}$.}
\end{align}
By (\ref{mu}), we have 
\begin{align}
\label{mu2}
\dot{\A}_0[j,j] = - \sum_{i\ne j}\SigDO[i,j]\dot{\A}_0[i,j] \ \text{for all $j\in\intSet{K}$.}
\end{align} 
Now notice that
\begin{align}
\label{inverse}
\frac{d \A_t^{-1}}{dt}\bigg|_{t=0} = - \A_0^{-1}\dot{\A}_0\A_0^{-1}
=  - \dot{\A}_0. 
\end{align}
Combining the above equality with Lemma \ref{lemma:diffNorm1} and \ref{lemma:diffNorm}, we have
\begin{align*}
\lim_{t\downarrow 0^+}\frac{1}{t}\left(\|\A^{-1}_t[j,S]\|_2 - \|\Iden[j,S]\|_2\right) = \left\{
\begin{array}{l l}
-\dot{\A}_0[j,j] & \quad \text{if $j\in S$} \\ 
\|\dot{\A}_0[j,S]\|_2 & \quad \text{if $j\not \in S$} 
\end{array}
\right.
\end{align*}
Therefore
\begin{align}
\label{deriv}
\lim_{t\downarrow 0^+}\frac{1}{t}\left(\vertiii{\A^{-1}_t[j,]}_s - \vertiii{\Iden[j,]}_s \right)= 
-\dot{\A}_0[j,j] + {K-1 \choose s-1}^{-1}\sum_{S:|S|=s, j\not\in S}\|\dot{\A}_0[j,S]\|_2.
\end{align}

\noindent Combining (\ref{plusForm}), (\ref{jsplit}), (\ref{mu2}) and (\ref{deriv}), we have 
\begin{align*}
\sqrt{\frac{\pi}{2}}\frac{K}{s}\Delta^+(L_{SG(s)}, \{\A_t\}_t)  & = 
-\sum_{j=1}^K\dot{\A}_0[j,j] + {K-1 \choose s-1}^{-1}\sum_{j}\sum_{S:|S|=s, j\not\in S}||\dot{\A}_0[j,S]||_2\\
& = \sum_{j=1}^K\Big(\sum_{i\ne j} \SigDO[i,j] \dot{\A}_0[j,i] + {K-1 \choose s-1}^{-1}\sum_{S:|S|=s, j\not\in S}\|\dot{\A}_0[j,S]\|_2\Big).
\end{align*}
Similarly, one can show 
\begin{align*}
\sqrt{\frac{\pi}{2}}\frac{K}{s}\Delta^-(L_{SG(s)}, \{\A_t\}_t)   = \sum_{j=1}^K\Big(\sum_{i\ne j} \SigDO[i,j] \dot{\A}_0[j,i] - {K-1 \choose s-1}^{-1}\sum_{S:|S|=s, j\not\in S}\|\dot{\A}_0[j,S]\|_2\Big).
\end{align*}
Thus for $s\in\intSet{K-1}$, to establish (\ref{eqn:gradientPM}) it suffices to require for each $j\in\intSet{K}$,
\begin{align}
\label{eq:basicInq}
 \Big|\sum_{i\ne j} \SigDO[i,j] \dot{\A}_0[j,i]\Big| < \frac{K-s}{K-1} {K-2 \choose s-1}^{-1}\sum_{S:|S|=s,j\not\in S}\|\dot{\A}_0[j,S]\|_2  = \frac{K-s}{K-1} \vertiii{\dot{\A}_0[j,-j]}_s.
\end{align}
for any $\dot{\A}_0$ such that $\dot{\A}_0[j,-j] \ne 0$. Since $\dot{\A}_0[j,i]$ is a free variable for $i\neq j$, (\ref{eq:basicInq}) is equivalent to
\[
\Big|\SigDO[-j,j]^T\w\Big| <  \frac{K-s}{K-1},
\]
for all $\w \in \mathbb R^{K-1}$ such that $\vertiii{\w}_s = 1$. Thus by the definition of the dual norm,  it suffices to have
\[
\vertiii{\SigDO[-j,j]}_s^* = \sup_{\vertiii{\w}_s = 1} \Big|\SigDO[-j,j]^T\w\Big|\ <  \frac{K-s}{K-1}.
\]
Therefore, the condition
\begin{align}
\label{eqn:suffSSparse}
\max_{1\leq j\leq K}\vertiii{\SigDO[-j,j]}_s^* < \frac{K-s}{K-1} = 1-\frac{s-1}{K-1}.
\end{align}
is sufficient for $\D_0$ to be locally identifiable with respect to the objective function $L_{SG(s)}$.

Similiarly, one can check that if the reversed strict inequality in (\ref{eqn:suffSSparse}) holds, $\D_0$ is not a local minimum of $L_{SG(s)}(\D)$. Thus we complete the proof for the $s$-sparse model.

\noindent (2) Now consider the Bernoulli($p$)-Gaussian model for $p\in (0,1)$. First of all, note that we have
\begin{align*}
& \sqrt{\frac{\pi}{2}}\frac{1}{p}\Delta^{\pm}(L_{BG(p)},\{\A_t\}_t) 
=  \sum_{j=1}^{K}\lim_{t\rightarrow 0^\pm} \frac{1}{t}\left( \vertiii{\A^{-1}_t[j,]}_p-\vertiii{\Iden[j,]}_p\right)  \\
= & \sum_{j=1}^K\bigg(\sum_{i\ne j} \SigDO[i,j] \dot{\A}_0[j,i] \pm (1-p)\sum_{k=0}^{K-2} p^{k}(1-p)^{K-2-k}\sum_{S:|S|=k+1, j\not\in S}\|\dot{\A}_0[j,S]\|_2\bigg) \\
= & \sum_{j=1}^K\bigg(\dot{\A}_0[j,-j]^T\SigDO[-j,j]  \pm (1-p)\sum_{k=0}^{K-2} \pbinom(k;K-2,p)  \vertiii{\dot{\A}_0[j,-j]}_{k+1}\bigg) \\
= & \sum_{j=1}^K\bigg(\dot{\A}_0[j,-j]^T\SigDO[-j,j] \pm (1-p)\vertiii{\dot{\A}_0[j,-j]}_{p}\bigg).
\end{align*}
Thus, similar to the $s$-sparse Gaussian case, it can be shown that a sufficient condition for local identifiability is
\[
\left|\SigDO[-j,j]^T\w\right| <  1 - p, 
\]
for all $j\in\intSet{K}$ and all $\w\in \mathbb R^{K-1}$ such that $\vertiii{\w}_p = 1$. The above condition is equivalent to
\[
\max_{1\leq j \leq K}\vertiii{\SigDO[-j,j]}_p^* < 1 - p.
\]
The rest of the proof can be proceeded as in the case of the $s$-sparse Gaussian model.

\noindent (3) Now let us consider the non-sparse case where $s = K$ or $p=1$. In this case, since the objective functions are the same under both models (see Theorem \ref{thm:popResult}), we only need to consider the $s$-sparse Gaussian model. If $s=K$, the RHS quantity in Inequality (\ref{eq:basicInq}) is zero. Thus,  the reference dictionary is not locally identifiable if
\[
\Big|\SigDO[-j,j]^T\w\Big| > 0,
\]
for some $j\in\intSet{K}$ and $\w\in\mathbb R^{K-1}$. Thus, if $\SigDO$ is not the identity matrix, or equivalently, if the reference dictionary $\D_0$ is not orthogonal, $\D_0$ is not locally identifiable. 

Next, let us deal with the case where $\D_0$ is orthogonal.  Let $\D\in\Dset$ be a full rank dictionary and $\W = \D^{-1}$. Since $\D_0$ is orthogonal, $\|\W[j,]\D_0\|_2 = \|\W[j,]\|_2$. By the fact that $\W\D=\Iden$ and $\|\D[,j]\|_2 = 1$, we have $1 = \W[j,]\D[,j] \leq \|\W[j,]\|_2\|\D[,j]\|_2 = \|\W[j,]\|_2$, where the equality holds iff $\W[j,]^T = \pm \D[,j]$. 

Under the $K$-sparse Gaussian model, 
\[
L_{SG(K)}(\D) = \sqrt{\frac{2}{\pi}}\sum_{j=1}^K\|\W[j,]\D_0\|_2 = \sqrt{\frac{2}{\pi}}\sum_{j=1}^K\|\W[j,]\|_2 \geq \sqrt{\frac{2}{\pi}}K = L_{SG(K)}(\D_0),
\]
where the equality holds for any $\D$ such that $\D^T\D=\I$. Thus, $L_{SG(K)}(\D_0) = L_{SG(K)}(\D_0\mathbf U)$ for any orthogonal matrix $\mathbf U\in\mathbb R^{K\times K}$, i.e. the objective function remains the same as we rotate $\D_0$. Therefore, $\D_0$ is not a local minimum of $L_{SG(K)}$.

In conclusion, $\D_0$ is not locally identifiable when $s=K$ or $p=1$.






\end{proof}

\subsection{Proofs of the finite sample results: Theorem \ref{thm:finiteSSparse} and Theorem \ref{thm:finiteBerGau}}
\label{proof:finiteResult}
\begin{proof}
 We will first recall the signal generation procedure in Section \ref{sec:prelim}. Let $\z$ be a $K$-dimensional standard Gaussian vector, and $\xibf\in \{0,1\}^K$ be either an $s$-sparse random vector or a Bernoulli random vector with probability $p$. Let $\z_1,...,\z_N$ and $\xibf_1,...,\xibf_N$ be identical and independent copies of $\z$ and $\xibf$ respectively. For each $i \in\intSet{N}$ and $j\in\intSet{K}$, define $\alphabf_i[j] = \z_i[j]\xibf_i[j]$. For $S\subset\intSet{K}$ with $1\leq |S| \leq K-1$, define
 \begin{align*}
\chi_i(S) = \left\{
\begin{array}{l l}
1 & \quad \text{if $\xibf_i[k]=1$ for all $k \in S$ and $\xibf_i[k]=0$ for all $k\in S^c$, } \\ 
0 & \quad \text{otherwise.} 
\end{array}
\right.
\end{align*}
On the other hand, if $S=\intSet{K}$, define $\chi_i(S) = 1$ if $\xibf_i[k]=1$ for all $k \in \intSet{K}$ and $\chi_i(S) = 0$ otherwise.  As in the population case, in the following analysis we will work with full rank dictionaries $\D$. First of all, notice that
\begin{align*}
l(\D,\x_i) = \|\D^{-1}\x_i\|_1 = \|\D^{-1}\D_0\alphabf_i\|_1 =
\sum_{j=1}^K\left|\A^{-1}[j,]\alphabf_i\right|
& =\sum_{j=1}^K\sum_{k=1}^K\left(\sum_{S:|S|=k}|\A^{-1}[j,S]\z_i[S]|\chi_i(S)\right).
\end{align*}
Next, we have
\begin{align}
\label{eq:threeParts}
\Delta^+(l(.,\x_i),\{\A_t\}_t)
& 
= \lim_{t\downarrow 0^+}
\frac{1}{t} \left(l(\D_0\A_t,\x_i) - l(\D_0,\x_i)\right)\nonumber \\
& = \sum_{j=1}^K
\bigg(-\sum_{k=1}^K\sum_{S:j\in S, |S|=k}\dot{\A}_0[j,j]|\z_i[j]|\chi_i(S) \nonumber \\
& \quad -\sgn(\z_i[j])\sum_{k=2}^K\sum_{S:j\in S, |S|=k} \sum_{l\in S, l\ne j}\dot{\A}_0[j,l]\z_i[l]\chi_i(S)  \nonumber \\
& \quad + \sum_{k=1}^{K-1}\sum_{S:j\not \in S, |S|=k}|\dot{\A}_0[j,S]\z_i[S]|\chi_i(S)\bigg).
\end{align}
Here $\sgn(x)$ is the sign function of $x\in\mathbb R$ such that $\sgn(x) = 1$ for $x>0$, $\sgn(x) = -1$ for $x<0$ and $\sgn(x) = 0$ for $x=0$. By (\ref{mu2}), the first term in (\ref{eq:threeParts}) can be rearranged as follows
\begin{align*}
-\sum_{j=1}^K|\z_i[j]|\sum_{k=1}^K\sum_{S:j\in S, |S|=k}\dot{\A}_0[j,j]\chi_i(S)
& = \sum_{j=1}^K|\z_i[j]|\sum_{k=1}^K\sum_{S:j\in S, |S|=k}\sum_{l\ne j}\SigDO[l,j]\dot{\A}_0[l,j]\chi_i(S) \\ 
& = \sum_{j=1}^K\sum_{l\ne j} \SigDO[j,l]\dot{\A}_0[j,l]\left(|\z_i[l]|\sum_{k=1}^K\sum_{S:l\in S, |S|=k}\chi_i(S)\right).
\end{align*}
The second term in (\ref{eq:threeParts}) can be rewritten as
\begin{align*}
-\sum_{j=1}^K \sgn(\z_i[j])\times\sum_{l\ne j}(\dot{\A}_0[j,l]\z_i[l])\times\sum_{k=2}^K\sum_{S:\{j,l\}\in S, |S|=k}\chi_i(S).
\end{align*}
For $j,l\in\intSet{K}$ such that $j\ne l$, define the following quantities 
\begin{align}
\label{Fi}
\LL_i[l,j] & = \SigDO[j,l]|\z_i[l]|\sum_{k=1}^K\sum_{S:l\in S, |S|=k}\chi_i(S), \\  
\label{Gi}
\QQ_i[l,j] & = \sgn(\z_i[j])\z_i[l]\sum_{k=2}^K\sum_{S:\{j,l\}\in S, |S|=s}\chi_i(S),
\end{align}
whereas $\LL[j,j] = \QQ[j,j] = 0$. For each $j\in\intSet{K}$, also define 
\begin{align}
\label{ti}
\TT_i[j](\w) & = \sum_{k=1}^{K-1}\sum_{S:j\not \in S, |S|=k}|\w[S]^T\z_i[S]|\chi_i(S).
\end{align}
Let $\bar{\LL}$, $\bar{\QQ}$ and $\bar{\TT}$ be the sample average of $\LL_i$, $\QQ_i$ and $\TT_i$ respectively. With the definitions (\ref{Fi}) -- (\ref{ti}), we have
\begin{align*}
\Delta^+ (L_N,\{\A_t\}_t)
& = \frac{1}{N}\sum_{i=1}^N \Delta^+( l(.,\x_i),\{\A_t\}_t)\\
& = \sum_{j=1}^K \frac{1}{N}\sum_{i=1}^N \left(\dot{\A}_0[j,]\LL_i[,j] + \dot{\A}_0[j,]\QQ_i[,j] + \TT_i[j](\dot{\A}_0[j,])\right) \\
& = \sum_{j=1}^K \left(\dot{\A}_0[j,]\bar{\LL}[,j] - \dot{\A}_0[j,]\bar{\QQ}[,j] + \bar{\TT}[j](\dot{\A}_0[j,])\right) 
\end{align*}
On the other hand, 
\[\Delta^-(L_N,\{\A_t\}_t)= \sum_{j=1}^K \left( \dot{\A}_0[j,]\bar{\LL}[,j] - \dot{\A}_0[j,]\bar{\QQ}[,j] - \bar{\TT}[j](\dot{\A}_0[j,])\right). 
\] 
Now for $j\in \intSet{K}$, $s\in \intSet{K-1}$ and $p\in(0,1)$, define 
\[\mathcal E_j(s) = \{\w\in\mathbb R^K, \vertiii{\w[-j]}_s= 1,\w[j]=0\},\] 
\[\mathcal F_j(p) = \{\w\in\mathbb R^K, \vertiii{\w[-j]}_p= 1,\w[j]=0\}.\] 
Thus to ensure that $\D_0$ is a local minimum, it suffices to have for each $j\in\intSet{K}$,
\[
\HH_j(\w) \defeq \left| \w^T \bar{\LL}[,j] - \w^T\bar{\QQ}[,j]\right| - \bar{\TT}[j](\w) < 0,
\]
for all $\w\in \mathcal E_j(s)$ for the $s$-sparse Gaussian model or all $\w\in \mathcal F_j(p)$ for the Bernoulli($p$)-Gaussian model. 

\noindent (1) For the $s$-sparse Gaussian model, let $j\in \intSet{K}$ and define
\[
\EH_j(\w) = \sqrt{\frac{2}{\pi}}\frac{s}{K}\left(\left|\w^T\SigDO[,j]\right|-\frac{K-s}{K-1}\right),
\]
which can be thought of as the expected value of $\HH_j(\w)$. Note that by triangle inequality,
\begin{align}
& \sup_{\w\in \mathcal E_j(s)}\left|\HH_j(\w) - \EH_j(\w)\right| \nonumber \\  
\leq & \sup_{\w\in \mathcal E_j(s)}\left|\w^T\left(\bar{\LL}[,j]-\sqrt{\frac{2}{\pi}}\frac{s}{K}\SigDO[,j] \right)\right| +
 \sup_{\w\in \mathcal E_j(s)} \left|\w^T\bar{\QQ}[,j]\right| + 
 \sup_{\w\in \mathcal E_j(s)} \left|\bar{\TT}[j](\w) - \sqrt{\frac{2}{\pi}}\frac{s}{K}\frac{K-s}{K-1}\right| \nonumber \\
\label{tri}
= & \vertiii{\bar{\LL}[-j,j]-\sqrt{\frac{2}{\pi}}\frac{s}{K}\SigDO[-j,j]}_s^* +
\vertiii{\bar{\QQ}[-j,j]}_s^* +
\sup_{\w\in \mathcal E_j(s)} \left|\bar{\TT}[j](\w) - \sqrt{\frac{2}{\pi}}\frac{s}{K}\frac{K-s}{K-1}\right|.
\end{align}
Thus, $\sup_{\w\in \mathcal E_j(s)}|\HH_j(\w) - \EH_j(\w)|>\frac{s}{K}\epsilon$ implies at least one of the three terms on the RHS is greater than $\frac{s}{K}\frac{\epsilon}{3}$. Using a union bound and by Lemma \ref{lemma:L}--\ref{lemma:T}, we have
\begin{align}
\label{eq:probJ}
\mathbb P \left\{ \sup_{\w\in \mathcal E_j(s)} |\HH_j(\w) - \EH_j(\w)|> \frac{s}{K}\epsilon \right\}
& \leq  2K\exp\left(-\frac{N\epsilon^2}{108K\|\SigDO[-j,j]\|_\infty}\right) \nonumber \\
& \quad \quad + 2K\exp\left(-\frac{s}{K} \frac{N\epsilon^2}{18(s/K)s+9\sqrt{2s}}\right) \nonumber\\
& \quad \quad + 3\left(\frac{24K}{\epsilon s}+1\right)^K\exp\left(-\frac{s}{K}\frac{N\epsilon^2}{360}\right). 
\end{align}
It is easy to see that the event $\left\{\sup_{\w\in \mathcal E_j(s)} |\HH_j(\w) - \EH_j(\w)|\leq \frac{s}{K}\epsilon\right\}$ implies
\begin{align}
\label{eq:infF}
\sup_{\w\in \mathcal E_j(s)} \EH_j(\w)-\frac{s}{K}\epsilon
\leq \sup_{\w\in \mathcal E_j(s)} \HH_j(\w)
\leq \sup_{\w\in \mathcal E_j(s)} \EH_j(\w)+\frac{s}{K}\epsilon.
\end{align}
On the other hand, 
\[\sup_{\w\in \mathcal E_j(s)} \EH_j(\w) = \sqrt{\frac{2}{\pi}}\frac{s}{K}\left( \vertiii{\SigDO[-j,j]}_s^* - \frac{K-s}{K-1} \right).\]
Thus, if $\vertiii{\SigDO[-j,j]}_s^* < \frac{K-s}{K-1} - \sqrt{\frac{\pi}{2}}\epsilon$, $\sup_{\w\in \mathcal E_j(s)}\HH_j(\w) < 0$ except with probability at most the bound in (\ref{eq:probJ}). To ensure $\D_0$ to be a local minimum, it suffices to have $\sup_{\w\in \mathcal E_j(s)}\HH_j(\w) < 0 $ for all $j\in\intSet{K}$. Thus, if $\vertiii{\SigDO[-j,j]}_s^* <  \frac{K-s}{K-1}-\sqrt{\frac{\pi}{2}}\epsilon$ for all $j\in\intSet{K}$, we have
\begin{align*}
\mathbb P \left\{ \text{$\D_0$ is locally identifiable} \right\}
&\geq  \mathbb P \left\{ \max_{j}\sup_{\w\in \mathcal E_j(s)}\HH_j(\w) < 0  \right\} \\ 
&\geq 1 - \mathbb P \left\{ \max_{j}\sup_{\w\in \mathcal E_j(s)}\HH_j(\w) \geq 0 \right\} \\ 
&\geq 1 - \sum_{j=1}^K \mathbb P \left\{ \sup_{\w\in \mathcal E_j(s)}\HH_j(\w) \geq 0  \right\} \\ 
&\geq 1 - \sum_{j=1}^K \mathbb P \left\{ \sup_{\w\in \mathcal E_j(s)} |\HH_j(\w) - \EH_j(\w)|> \frac{s}{K}\epsilon \right\} \\ 
& \geq 1 - 2K^2\exp\left(-\frac{N\epsilon^2}{108K\max_{l\ne j}|\SigDO[l,j]|}\right) \\
& \quad \quad - 2K^2\exp\left(-\frac{s}{K} \frac{N\epsilon^2}{18(s/K)s+9\sqrt{2s}}\right) \\
& \quad \quad - 3K\left(\frac{24K}{\epsilon s}+1\right)^K\exp\left(-\frac{s}{K}\frac{N\epsilon^2}{360}\right). 
\end{align*}
On the other hand, to ensure $\D_0$ is not locally identifiable with high probability, it suffices to have $\vertiii{\SigDO[-j,j]}_s^* > \frac{K-s}{K-1}+\sqrt{\frac{\pi}{2}}\epsilon$ for some $j\in\intSet{K}$. Indeed, under that condition, the LHS inequality in (\ref{eq:infF}) implies $\sup_{\w\in \mathcal E_j(s)}\HH_j(\w) > 0$. Therefore 
\begin{align*}
\mathbb P \left\{ \text{$\D_0$ is not locally identifiable} \right\}
&\geq \mathbb P \left\{\sup_{\w\in \mathcal E_j(s)}\HH_j(\w) > 0  \right\} \\
&\geq 1 - \mathbb P \left\{\sup_{\w\in \mathcal E_j(s)}\HH_j(\w) \leq 0 \right\} \\
&\geq 1 - \mathbb P \left\{ \sup_{\w\in \mathcal E_j(s)} |\HH_j(\w) - \EH_j(\w)|> \frac{s}{K}\epsilon \right\} \\ 
& \geq 1-  2K\exp\left(-\frac{N\epsilon^2}{108K\|\SigDO[-j,j]\|_\infty}\right) \\
& \quad \quad - 2K\exp\left(-\frac{s}{K} \frac{N\epsilon^2}{18(s/K)s+9\sqrt{2s}}\right) \\
& \quad \quad - 3\left(\frac{24K}{\epsilon s}+1\right)^K\exp\left(-\frac{s}{K}\frac{N\epsilon^2}{360}\right).
\end{align*}
\noindent (2) For the Bernoulli($p$)-Gaussian model, define
\[
\nu_j(\w) = \sqrt{\frac{2}{\pi}}p\left(\left|\w^T\SigDO[,j]\right|-(1-p)\right).
\]
\noindent Similar to (\ref{tri}), by triangle inequality,
\begin{align*}
& \sup_{\w\in \mathcal F_j(p)}\left|\HH_j(\w) - \nu_j(\w)\right| \\  
\leq &\vertiii{\bar{\LL}[-j,j]-\sqrt{\frac{2}{\pi}}p\SigDO[-j,j]}_p^* +
\vertiii{\bar{\QQ}[-j,j]}_p^* +
\sup_{\w\in \mathcal F_j(p)} \left|\bar{\TT}[j](\w) - \sqrt{\frac{2}{\pi}}p(1-p)\right|.
\end{align*}
Then the analysis can be carried out in a similar manner using the parallel version of the concentration inequalities, i.e. Part 2 of Lemma \ref{lemma:L}--\ref{lemma:T}.
\end{proof}

\subsection{Concentration inequalities}
We will make frequent use of the following version of Bernstein's inequality. The proof of the inequality can be found in, e.g. Chapter 14 of \cite{Buhlmann2011}.
\begin{theorem}(Bernstein's inequality)
Let $Y_1,...,Y_N$ be independent random variables that satisfy the moment condition
\[
\mathbb E Y_i^m \leq \frac{1}{2}\times V \times m! \times B^{m-2}, 
\] 
for integers $m\geq 2$. Then
\[
\mathbb P \left\{ \frac{1}{N}|\sum_{i=1}^N Y_i - \mathbb E Y_i| > \epsilon \right\}
\leq 2\exp\left(-\frac{N\epsilon^2}{2V+2B\epsilon}\right).
\]
\end{theorem}

\begin{lemma}(Uniform concentration of $\bar{\LL}[-j,j]$)
\label{lemma:L} For $i\in \intSet{N}$, let $\LL_i \in \mathbb R^{K\times K}$ be defined as in (\ref{Fi}) and $\bar{\LL} = (1/N) \sum_{i=1}^N\LL_i$.
\begin{enumerate}
\item Under the $s$-sparse Gaussian model with $s\in\intSet{K-1}$,
\[
\mathbb P\left\{
\vertiii{\bar{\LL}[-j,j]-\sqrt{\frac{2}{\pi}}\frac{s}{K}\SigDO[-j,j]}_s^* > \frac{s}{K}\epsilon
\right\}
\leq  2K\exp\left(-\frac{N\epsilon^2}{12K\|\SigDO[-j,j]\|_\infty}\right),
\]
for $0 < \epsilon \leq 1$.
\item Under the Bernoulli-Gaussian model with parameter $p\in(0,1)$,
\[
\mathbb P\left\{
\vertiii{\bar{\LL}[-j,j]-\sqrt{\frac{2}{\pi}}p\SigDO[-j,j]}_p^* > p\epsilon
\right\}
 \leq 
2K\exp\left(-\frac{N\epsilon^2}{12(K+2p^{-1})\|\SigDO[-j,j]\|_\infty}\right),
\]
for $0 < \epsilon \leq 1$.
\end{enumerate}
In particular, if $\|\SigDO[-j,j]\|_\infty=0$, then the RHS bound is trivially zero.
\end{lemma}
\begin{proof} (1)
First of all, we will prove the inequality for the $s$-sparse model. Notice that by Lemma \ref{lemma:normBounds}, we have
\begin{align*}
\vertiii{\bar{\LL}[-j,j]-\sqrt{\frac{2}{\pi}}\frac{s}{K}\SigDO[-j,j]}_s^* 
\leq & \max_{ |S|=s,j\not\in S} \| \bar{\LL}[S,j]-\sqrt{\frac{2}{\pi}}\frac{s}{K}\SigDO[S,j]\|_2 \\
\leq & \sqrt{s} \max_{l\ne j} | \bar{\LL}[l,j]-\sqrt{\frac{2}{\pi}}\frac{s}{K}\SigDO[l,j]|.
\end{align*}
For convenience, define
\[\vv_i[l] = |\z_i[l]|\sum_{k=1}^K\sum_{|S|=k, l\in S}\chi_i(S) - \sqrt{\frac{2}{\pi}}\frac{s}{K}.\]
for $i\in\intSet{N}$ and $l\in\intSet{K}$. Note that $\sum_{k=1}^K\sum_{l\in S,|S|=k}\chi_i(S) = 1$ with probability ${K \choose s}^{-1}{K-1\choose s-1} = \frac{s}{K}$. Thus 
\[
\mathbb E \left(\sum_{k=1}^K\sum_{|S|=k,l\in S}\chi_i(S)\right)^m = \frac{s}{K}.
\]
For $m\geq 1$, by Jensen's inequality $|\frac{a+b}{2}|^m\leq \frac{1}{2}(|a|^m+|b|^m)$ and $\mathbb E |Z|^m \geq (\mathbb E |Z|)^m = \left(\frac{2}{\pi}\right)^{\frac{m}{2}}$, where $Z$ is a standard Gaussian variable. In addition, $\mathbb E |Z|^m \leq (m-1)!!\leq 2^{-\frac{m}{2}}m!$. Hence
\begin{align*}
\mathbb E |\vv_{i}[l]|^m & \leq 2^{m-1}\left(\mathbb E |\z_i[l]|^m+ \left(\frac{2}{\pi}\right)^{\frac{m}{2}}\left(\frac{s}{K}\right)^{m}\right) \\
& \leq  2 \times \mathbb E |Z|^m \times 2^{m-1}\\
& \leq  2 \times \left(\frac{1}{2}\right)^{\frac{m}{2}} m! \times 2^{m-1}\\
& = \frac{1}{2}\times \frac{4s}{K}\times m!\times(\sqrt{2})^{m-2}.
\end{align*}
Thus by Bernstein's inequality, we have
\[
\mathbb P\left\{
\left|\frac{1}{N}\sum_{i=1}^N \vv_i[l]\right| > \epsilon \right\}
\leq 2\exp\left(-\frac{N\epsilon^2}{2(4\frac{s}{K}+\sqrt{2}\epsilon)}\right).
\]
Therefore,
\begin{align*}
\mathbb P\left\{
\left|\SigDO[j,l]\frac{1}{N}\sum_{i=1}^N \vv_i[l]\right| > \frac{s}{K}\epsilon \right\}
& \leq 2\exp\left(-\frac{s}{K}\frac{N\epsilon^2}{2(4\SigDO[j,l]^2+\sqrt{2}\left|\SigDO[j,l]\right|\epsilon)}\right) \\
& \leq 2\exp\left(-\frac{s}{K}\frac{N\epsilon^2}{2\left|\SigDO[j,l]\right|(4+\sqrt{2}\epsilon)}\right) \\
& \leq 2\exp\left(-\frac{s}{K}\frac{N\epsilon^2}{12\left|\SigDO[j,l]\right|}\right).
\end{align*}
for $\epsilon \leq 1$. Notice that if $\SigDO[j,l]=0$ the LHS probability is trivially zero. Using a union bound, we have
\begin{align*}
\mathbb P\left\{\| \bar{\LL}[-j,j]-\sqrt{\frac{2}{\pi}}\frac{s}{K}\SigDO[-j,j]\|_\infty> \frac{s}{K}\epsilon\right\}& = 
\mathbb P\left\{\max_{l\ne j}|\SigDO[j,l]\frac{1}{N}\sum_{i=1}^N  \vv_i[l]| > \epsilon \right\} \\ 
& \leq  2K\exp\left(-\frac{s}{K}\frac{N\epsilon^2}{12\|\SigDO[-j,j]\|_\infty}\right).
\end{align*}
Therefore
\begin{align*}
\mathbb P\left\{
\vertiii{\bar{\LL}[-j,j]-\sqrt{\frac{2}{\pi}}\frac{s}{K}\SigDO[-j,j]}_s^* > \frac{s}{K}\epsilon
\right\}
& \leq
\mathbb P\left\{ \sqrt{s}\|\bar{\LL}[,j]-\sqrt{\frac{2}{\pi}}\frac{s}{K}\SigDO[,j]\|_\infty>\frac{s}{K}\epsilon\right\} \\
& \leq  2K\exp\left(-\frac{N\epsilon^2}{12K\|\SigDO[-j,j]\|_\infty}\right).
\end{align*}
\noindent (2) Now let us consider the Bernoulli-Gaussian model. Notice that by Lemma \ref{lemma:normRelation}, for $\frac{s-1}{K-1}\geq p$, we have
\begin{align*}
\vertiii{\bar{\LL}[-j,j]-\sqrt{\frac{2}{\pi}}p\SigDO[-j,j]}_p^*
& \leq \vertiii{\bar{\LL}[-j,j]-\sqrt{\frac{2}{\pi}}p\SigDO[-j,j]}_s^* \\
& \leq \sqrt{s} \left\| \bar{\LL}[-j,j]-\sqrt{\frac{2}{\pi}}p\SigDO[-j,j]\right\|_\infty.
\end{align*}
Now let $s = \lceil pK - p + 1\rceil\leq pK +2$. For $i\in\intSet{N} $ and $l\in\intSet{K}$, define
\[\uu_i[l] = \left|\z_i[l]\right|\sum_{k=1}^K\sum_{|S|=s, l\in S}\chi_i(S) - \sqrt{\frac{2}{\pi}}p.\]
Note that the event $\left\{\sum_{k=1}^K\sum_{|S|=k,l\in S}\chi_i(S) = 1\right\}$ is the same as the event that $\left\{\alphabf_i[l] = 1\right\}$, which, happens with probability $p$. Thus
\[
\mathbb E \left(\sum_{k=1}^K\sum_{|S|=k,l\in S}\chi_i(S)\right)^m = p.
\]
Similar to the case of $s$-sparse model, 
\begin{align*}
\mathbb E \left|\uu_{i}[l]\right|^m 
\leq \frac{1}{2}\times 4p\times m!\times\left(\sqrt{2}\right)^{m-2}.
\end{align*}
By Bernstein's inequality, we have
\[
\mathbb P\left\{
\left|\frac{1}{N}\sum_{i=1}^N \uu_i[l]\right| > \epsilon \right\}
\leq 2\exp\left(-\frac{N\epsilon^2}{2(4p+\sqrt{2}\epsilon)}\right).
\]
Therefore
\begin{align*}
\mathbb P\left\{
\vertiii{\bar{\LL}[-j,j]-\sqrt{\frac{2}{\pi}}p\SigDO[-j,j]}_p^* > p\epsilon
\right\}
& \leq
\mathbb P\left\{ \sqrt{s}\|\bar{\LL}[,j]-\sqrt{\frac{2}{\pi}}p\SigDO[,j]\|_\infty>p\epsilon\right\} \\
& \leq 2K\exp\left(-\frac{p}{s}\frac{N\epsilon^2}{2\|\SigDO[-j,j]\|_\infty(4+\sqrt{2}\epsilon)}\right) \\
& \leq 2K\exp\left(-\frac{N\epsilon^2}{12(K+2p^{-1})\|\SigDO[-j,j]\|_\infty}\right),
\end{align*}
for $\epsilon \leq 1$.
\end{proof}

\begin{lemma}(Uniform concentration of $\bar{\QQ}[-j,j]$) For $i\in \intSet{N}$, let $\QQ_i \in \mathbb R^{K\times K}$ be defined as in (\ref{Gi}) and $\bar{\QQ} = (1/N) \sum_{i=1}^N\QQ_i$.
\label{lemma:Q}
\begin{enumerate}
\item Under the $s$-sparse Gaussian model with $s\in\intSet{K-1}$,
\[
\mathbb P\left\{
\vertiii{\bar{\QQ}[-j,j]}_s^* >\frac{s}{K}\epsilon \right\}
 \leq 2K\exp\left(-\frac{s}{K} \frac{N\epsilon^2}{2(s/K)s+\sqrt{2s}}\right), 
\]
for $0<\epsilon\leq 1$.
\item Under the Bernoulli-Gaussian model with parameter $p\in(0,1)$,
\[
\mathbb P\left\{
\vertiii{\bar{\QQ}[-j,j]}_p^* >p\epsilon \right\}
\leq 2K\exp\left(-p \frac{N\epsilon^2}{p(pK+2)+\sqrt{2(pK+2)}}\right),
\]
for $0<\epsilon\leq 1$.
\end{enumerate}
\end{lemma}
\begin{proof}
The proof is highly similar to that of Lemma \ref{lemma:L} and so we will omit some common steps.

\noindent (1) We first prove the concentration inequality for the $s$-sparse model. Notice that
\[
\vertiii{\bar{\QQ}[-j,j]}_s^* \leq \sqrt{s} \max_{l \ne j}| \bar{\QQ}[l,j] |.
\]
In addition,
\begin{align*}
\mathbb E \left(\sum_{k=2}^K\sum_{\{j,l\}\in S,|S|=k}\chi_i(S)\right)^m
& = \mathbb E \left(\sum_{k=2}^K\sum_{|S|=k,\{j,l\}\in S}\chi_i(S)\right) \\
& = {K \choose s}^{-1}{K-2 \choose s-2} = \frac{s(s-1)}{K(K-1)}
\leq (\frac{s}{K})^2.
\end{align*}
Thus
\[
\mathbb E |\QQ_i[l,j]|^m
\leq  2^{-m/2}m! \times (\frac{s}{K})^2
= \frac{1}{2}\times (\frac{s}{K})^2 \times m! \times (\frac{1}{\sqrt{2}})^{m-2}.
\]
By Bernstein inequality:
\[
\mathbb P \left\{|\frac{1}{N}\sum_{i=1}^N \QQ_i[l,j]|>\epsilon\right\}
\leq 2\exp\left(-\frac{N\epsilon^2}{2(s/K)^2+\sqrt{2}\epsilon)}\right).
\]
Thus we have
\begin{align*}
\mathbb P\left\{
\vertiii{\bar{\QQ}[-j,j]}_s^* >\frac{s}{K}\epsilon \right\}
& \leq \mathbb P\left\{
\sqrt{s}\max_{l \ne j}| \bar{\QQ}[l,j] | > \frac{s}{K}\epsilon \right\} \\
& \leq 2K\exp\left(-\frac{(s/K)^2 N(\epsilon^2/s)}{2(s/K)^2+\sqrt{2}(s/K)(\epsilon/\sqrt{s}))}\right) \\
& \leq 2K\exp\left(-\frac{s}{K} \frac{N\epsilon^2}{2(s/K)s+\sqrt{2s}\epsilon}\right) \\
& \leq 2K\exp\left(-\frac{s}{K} \frac{N\epsilon^2}{2(s/K)s+\sqrt{2s}}\right), 
\end{align*}
for $\epsilon\leq 1$.
 
\noindent (2) For Bernoulli-Gaussian model, notice that
\[
\vertiii{\bar{\QQ}[-j,j]}_p^* \leq \vertiii{\bar{\QQ}[-j,j]}_s^* 
\leq \sqrt{s} \max_{l \ne j}| \bar{\QQ}[l,j]|,
\]
for $s = \lceil pK - p + 1\rceil\leq pK +2$. Also,
\[
\mathbb E |\QQ_i[l,j]|^m
\leq 2^{-m/2}m! \times p^2
= \frac{1}{2}\times p^2 \times m! \times (\frac{1}{\sqrt{2}})^{m-2}.
\]
Thus
\begin{align*}
\mathbb P\left\{
\vertiii{\bar{\QQ}[-j,j]}_s^* >p\epsilon \right\}
& \leq \mathbb P\left\{
\sqrt{s}\max_{l \ne j}| \bar{\QQ}[l,j] | > p\epsilon \right\} \\
& \leq 2K\exp(-p \frac{N(\epsilon^2)}{2ps+\sqrt{2s}}) \\
& \leq 2K\exp\left(-p \frac{N\epsilon^2}{p(pK+2)+\sqrt{2(pK+2)}}\right),
\end{align*}
for $\epsilon\leq 1$.
\end{proof}

\begin{lemma}(Uniform concentration of $\bar{\TT}[j](\w)$)
\label{lemma:T} For $i\in \intSet{N}$, let $\TT_i$ be a function from $\mathbb R^{K}$ to $\mathbb R^{K}$ defined as in (\ref{ti}) and $\bar{\TT} = (1/N) \sum_{i=1}^N\TT_i$. Recall that for $j\in \intSet{K}$, $s\in \intSet{K-1}$ and $p\in(0,1)$,  
\[\mathcal E_j(s) = \{\w\in\mathbb R^K, \vertiii{\w[-j]}_s= 1,\w[j]=0\},\] 
\[\mathcal F_j(p) = \{\w\in\mathbb R^K, \vertiii{\w[-j]}_p= 1,\w[j]=0\}.\] 
\begin{enumerate}
\item Under the $s$-sparse Gaussian model with $s\in\intSet{K-1}$,
\[
\mathbb P\left\{
\sup_{w\in \mathcal E_j(s)} |\bar{\TT}[j](\w) - \sqrt{\frac{2}{\pi}}\frac{s}{K}\frac{K-s}{K-1}| > \frac{s}{K}\epsilon \right\}
\leq 3\left(\frac{8K}{\epsilon s}+1\right)^K\exp\left(-\frac{s}{K}\frac{N\epsilon^2}{40}\right),
\]
for $0 <\epsilon \leq \frac{1}{2}$.
\item Under the Bernoulli-Gaussian model with parameter $p\in(0,1)$,
\[
\mathbb P\left\{
\sup_{w\in \mathcal F_j(p)} |\bar{\TT}[j](\w) - \sqrt{\frac{2}{\pi}}p(1-p)| > p\epsilon \right\}
\leq 3\left(\frac{8}{\epsilon p}+1\right)^K\exp\left(-p\frac{N\epsilon^2}{40}\right),
\]
for $0 <\epsilon \leq \frac{1}{2}$.
\end{enumerate}
\end{lemma}

\begin{proof}
(1) Under the $s$-sparse model, we have
\begin{align*}
\mathbb E |\TT_i[j](\w)|^m 
& = \mathbb E \left(\sum_{|S|=s,j\not \in S}|\w[S]^T\z_i[S]|\chi_i(S)\right)^m \\
& = \sum_{|S|=s,j\not \in S}\mathbb E |\w[S]^T\z_i[S]|^m \mathbb E \chi_i(S)\\
& = {K \choose s}^{-1}\sum_{|S|=s,j\not \in S}\mathbb E |\w[S]^T\z_i[S]|^m.
\end{align*}
Notice that we have used the facts that the events $\chi_i(S)$'s are mutually exclusive and that $\z_i[S]$ and $\chi_i(S)$ are independent. Since the random variable $\w[S]^T\z_i[S]$ has distribution $N(0,\|\w[S]\|_2)$, $\mathbb E |\w[S]^T\z_i[S]|^m = \|\w[S]\|_2^m\mathbb E |Z|^m\leq 2^{-\frac{m}{2}}m!$. Therefore
\[
\mathbb E |\TT_i[j](\w)|^m 
\leq 2^{-\frac{m}{2}}m! {K\choose s}^{-1} \sum_{j\not \in S,|S|=s} \|\w[S]\|_2^m.
\]
Note that by Lemma \ref{lemma:normDecreasing}, $\vertiii{\w[-j]}_s\geq \|\w[-j]\|_2 \geq \|\w[S]\|_2$ for all $S$ such that $j\not \in S$. For $\w\in \mathcal E_j(s)$, $\vertiii{\w}_s = 1$ and so $\|\w_S\|_2\leq1$, which, further implies that $\|\w[S]\|_2^m\leq \|\w[S]\|_2$. Thus we have
\begin{align*}
\mathbb E |\TT_i[j](\w)|^m 
& \leq  2^{-\frac{m}{2}}m! {K\choose s}^{-1} \sum_{j\not \in S,|S|=s} \|\w[S]\|_2 \\
& \leq  2^{-\frac{m}{2}}m! \frac{s(K-s)}{K(K-1) } \vertiii{\w[-j]}_s \\
& =  2^{-\frac{m}{2}}m! \frac{s(K-s)}{K(K-1)} 
\end{align*}
For a fixed $j$, define
\[
U_{i}(\w) = \TT_i[j](\w) - \sqrt{\frac{2}{\pi}}\frac{s}{K}\frac{K-s}{K-1}.
\]
Notice that $\mathbb E U_i(\w) = 0$. In addition,
\begin{align*}
\mathbb E |U_{i}(\w)|^m\leq & 2^m \mathbb E |\TT_i[j](\w)|^m 
\leq \frac{1}{2}\times 4\frac{s}{K}\frac{K-s}{K-1} \times m! \times (\sqrt{2})^{m-2}.
\end{align*}
By Bernstein's inequality
\begin{align*}
\mathbb P\left\{\frac{1}{N}|\sum_{i=1}^N U_{i}(\w)| > \frac{s}{K}\epsilon \right\} & \leq 2\exp\left(-\frac{s}{K}\frac{N\epsilon^2}{2(4\frac{K-s}{K-1} + \sqrt{2}\epsilon)}\right)
\leq 2 \exp \left( -\frac{s}{K}\frac{N\epsilon^2}{10}\right),
\end{align*}
for $0<\epsilon\leq 1/2$. Now let $\{\w_i\}$ be an $\delta$-cover of $\mathcal E_j(s)$. Since $\mathcal E_j(s)$ is contained in the unit ball $\{\w\in \mathbb R^{K-1}: \|\w\|_2\leq1\}$, there exists a cover such that $|\{\w_l\}|\leq \left(\frac{2}{\delta}+1\right)^{K-1}$. For any $\w,\w'\in \mathcal E_j(s)$, we have
\begin{align*}
\left|U_i(\w) - U_i(\w')\right|
\leq & \sum_{j\not\in S, |S|=s}\left|(\w[S]-\w'[S])^T\z_i[S]\right|\chi_i(S).
\end{align*}
Let $Z$ be a standard Gaussian variable. We have
\begin{align*}
\mathbb P \left\{
\sum_{|S|=s,j\not\in S} \left|\w[S]^T\z_i[S]\right|\chi_i(S) > \epsilon
\right\}
 & = {K-1 \choose s}^{-1} \sum_{|S|=s,j\not\in S} \mathbb P \left\{ \left|\w[S]^T\z_i[S]\right| > \epsilon \right\} \\
 & = {K-1 \choose s}^{-1} \sum_{|S|=s,j\not\in S} \mathbb P \left\{ \|\w[S]\|_2|Z| > \epsilon \right\} \\
& \leq \mathbb P \left\{ \|\w\|_2|Z| > \epsilon \right\}.
\end{align*}
Let $Z_i$, $i=1,...,N$, be $i.i.d$ standard Gaussian variables. By the one-sided Bernstein's inequality,
\begin{align*}
\mathbb P\left\{\frac{1}{N}\sum_{i=1}^N|Z_i|\geq 2\right\} \leq \exp\left(-\frac{N(2-\sqrt{2/\pi})^2}{2(4+\sqrt{2}(2-\sqrt{2/\pi})}\right)
\leq \exp\left(-\frac{N}{8}\right).
\end{align*}
Now let $\delta = \frac{s}{K}\frac{\epsilon}{4}$. Thus
\begin{align*}
 \mathbb P \left\{ \sup_{\|\w-\w'\|_2\leq \delta}|\frac{1}{N}\sum_{i=1}^N(U_i(\w)-U_i(\w'))|> \frac{s}{K}\frac{\epsilon}{2}\right\} 
 & \leq  \mathbb P \left\{ \sup_{\|\w'-\w\|_2\leq \delta} \frac{1}{N}\sum_{i=1}^N\left|U_i(\w)-U_i(\w')\right| > \frac{s}{K}\frac{\epsilon}{2} \right\} \\
& \leq  \mathbb P \left\{ \sup_{\|\w'-\w\|_2\leq \delta} \frac{1}{N}\sum_{i=1}^N\|\w - \w'\|_2 |Z_i| > \frac{s}{K}\frac{\epsilon}{2} \right\} \\
& \leq  \mathbb P \left\{ \delta \frac{1}{N}\sum_{i=1}^N |Z_i| > \frac{s}{K}\frac{\epsilon}{2}\right\}
\leq \mathbb P \left\{ \frac{1}{N}\sum_{i=1}^N |Z_i| > 2 \right\}\\
& \leq  \exp\left(-\frac{N}{8}\right).
\end{align*}
 By triangle inequality
\[
\sup_{\|\w'-\w\|_2\leq \delta}|\frac{1}{N}\sum_{i=1}^N U_i(\w')|
\leq \sup_{\|\w'-\w\|_2\leq \delta}|\frac{1}{N}\sum_{i=1}^N(U_i(\w)-U_i(\w'))|+|\frac{1}{N}\sum_{i=1}^N U_i(\w)|.
\]
Using a union bound, we have
\begin{align*}
\mathbb P \{ \sup_{\|\w'-\w\|_2\leq \delta} \left|\frac{1}{N}\sum_{i=1}^N U_i(\w')\right|
> \frac{s}{K}\epsilon \}
& \leq \mathbb P \left\{ \sup_{\|\w-\w'\|_2\leq \delta}\left|\frac{1}{N}\sum_{i=1}^N(U_i(\w)-U_i(\w'))\right|> \frac{s}{K}\frac{\epsilon}{2}\right\} \\ 
& \quad \quad +\mathbb P \left\{ \left|\frac{1}{N}\sum_{i=1}^N U_i(\w)\right| > \frac{s}{K}\frac{\epsilon}{2} \right\} \\
& \leq \exp\left(-\frac{N}{8}\right) + 2 \exp \left( -\frac{s}{K}\frac{N\epsilon^2}{40}\right) \\
& \leq 3 \exp \left( -\frac{s}{K}\frac{N\epsilon^2}{40}\right),
\end{align*}
for $0<\epsilon\leq 1$. Now apply union bound again,
\begin{align*}
\mathbb P\left\{\sup_{\w\in \mathcal E_j(s)} \frac{1}{N}|\sum_{i=1}^N U_{i}(\w)| > \frac{s}{K}\epsilon \right\} 
& \leq \mathbb P\left\{\max_{l} \sup_{||\w-\w_l||_2\leq \delta} \frac{1}{N}|\sum_{i=1}^N U_{i}(\w)| > \frac{s}{K}\epsilon \right\} \\
& \leq 3\left(\frac{8K}{\epsilon s}+1\right)^K\exp\left(-\frac{s}{K}\frac{N\epsilon^2}{40}\right).
\end{align*}
\noindent (2) For $\w\in \mathcal F_j(p)$, under the Bernoulli-Gaussian model:
\begin{align*}
\mathbb E \left|\TT_i[j](\w)\right|^m 
& = \mathbb E |Z|^m \sum_{k=1}^{K-1} \sum_{|S|=k,j\not \in S} \|\w[S]\|_2^m \times p^k(1-p)^{K-k} \\
& \leq \mathbb E |Z|^m p \sum_{k=1}^{K-1} \sum_{|S|=k,j\not \in S} \|\w[S]\|_2 \times p^{k-1}(1-p)^{K-k} \\
& =  \mathbb E |Z|^m p(1-p) \sum_{k=0}^{K-2} \sum_{|S|=k+1,j\not \in S} \|\w[S]\|_2 \times p^{k}(1-p)^{K-2-k} \\
&=  \mathbb E |Z|^m p(1-p) \vertiii{\w[-j]}_p 
 =  \mathbb E |Z|^m p(1-p) \\
& \leq  2^{-m/2}m! p(1-p).
\end{align*}
Notice that we have used the fact that $\|\w[S]\|_2 \leq \|\w[-j]\|_2 \leq \vertiii{\w[-j]}_p = 1 $ for all $S$ such that $j\not \in S$. For each fixed $\w$, define
\[
V_{i}(\w) = \TT_i[j](\w) - \sqrt{\frac{2}{\pi}}(1-p)p.
\]
Now we have
\begin{align*}
\mathbb E |V_{i}(\w)|^m\leq & 2^m \mathbb E |\TT_i[j](\w)|^m 
\leq \frac{1}{2}\times 4p(1-p) \times m! \times (\sqrt{2})^{m-2}.
\end{align*}
The remaining parts of the proof can be preceeded exactly as in the case of the $s$-sparse model, noticing that we only need to replace $\frac{s}{K}$ by $p$, and $\frac{K-s}{K-1}$ by $1-p$.

\end{proof}

\subsection{Dual analysis of $\vertiii{.}_s$ and $\vertiii{.}_p$}
In this section, we will characterize the dual norms $\vertiii{.}_s^*$ and $\vertiii{.}_p^*$ by second order cone programs (SOCP). The characterization is helpful for deriving bounds for these special norms in the next section.
\begin{lemma} 
\label{lemma:dualNorm}
For $i \in \intSet{M}$, let $\A_i$ be an $k_i\times K$ with rank $k_i$. For $\z \in \mathbb R^K$, define
\[\|\z\|_{\A} = \sum_{i=1}^M \|\A_i\z\|_2.\]
Then the dual norm of $\|.\|_{\A}$ is
\[
\|\vv\|_{\A}^*
= \inf\left\{\max_i{\|\y_{i}\|_2}, \y_{i}\in \mathbb R^{k_i}, \sum_{i=1}^M \A_i^T\y_i = \vv\right\}.
\]
\end{lemma}
\begin{proof}
\[\|\vv\|_\A^* = \sup_{\z\ne0}\frac{\vv^T\z}{\|\z\|_{\A}} = \sup\left\{\vv^T\z: \|\z\|_\A\leq 1\right\}.\]
Introducing Lagrange multiplier $\lambda\geq0$ for the inequality constraint, the above problem is equivalent to the following
\begin{align*}
\|\vv\|_\A^* & = \sup_\z\left\{\inf_{\lambda\geq0} \Big\{ \vv^T\z +\lambda (1-\|\z\|_\A) \Big\}\right\} \\
& = \sup_\z\left\{\inf_{\lambda\geq0}\Big\{ \vv^T\z +\lambda (1-\sum_{i=1}^M \|\A_i\z\|_2)\Big\} \right\}.
\end{align*}
The dual problem is
\[
d = \inf_{\lambda\geq0}\left\{ \sup_\z\Big\{\vv^T\z +\lambda (1-\sum_{i=1}^M \|\A_i\z\|_2)\Big\}\right\}.
\]
Notice that $\|\A_i\z\|_2 = \sup\{\z^T\A_i^T\uu_i: \|\uu_i\|_2\leq1\}$. Hence 
\[
d = \inf_{\lambda\geq0}\left\{ \lambda + \sup_{\z,\uu}\Big\{\z^T(\vv -\lambda\sum_{i=1}^M \A_i^T\uu_i): \|\uu_i\|_2\leq 1\Big\}\right\}.
\]
Since the vector $\z$ can be arbitrary, in order to have a finite value, we must have $\lambda \sum_{i=1}^M \A_i^T\uu_i = \vv$. Now let $\y_i = \lambda \uu_i$, the problem becomes
\[
d = \inf_{\lambda\geq0}\left\{ \lambda: \sum_{i=1}^M \A_i^T\y_i = \vv, \|\y_i\|_2\leq \lambda\right\}.
\]
The above problem is exactly equivalent to
\[\inf\left\{\max_i{\|\y_{i}\|_2}, \y_{i}\in \mathbb R^{k_i}, \sum_{i=1}^M \A_i^T\y_i = \vv\right\}.
\]
Finally, notice that the original problem is convex and strictly feasible. Thus Slater's condition holds and the duality gap is zero. Hence
\[
||\vv||_\A^*
= \inf\left\{\max_i{\|\y_{i}\|_2}, \y_{i}\in \mathbb R^{k_i}, \sum_{i=1}^M \A_i^T\y_i = \vv\right\}.
\]
\end{proof}

The following corollary gives an alternative characterization of $\vertiii{.}_s$ and $\vertiii{.}_p$:
\begin{corollary} 
\label{col:dualNorm}
Denote by $\ys\in\mathbb R^{|S|}$ a variable vector indexed by the set $S$ (as opposed to being a subvector of $\y$). For $\z\in \mathbb R^{m}$, we have
\[\vertiii{\z}_s^*=\inf\left\{\max_{|S|=s}{\|\ys\|_2}: \ys\in \mathbb R^{s}, \sum_{|S|=s} \E_S^T\ys = \z \right\},\]
and
\[\vertiii{\z}_p^* = \inf\left\{\max_{S}{\|\ys\|_2}: \ys\in \mathbb R^{|S|}, \sum_{k=0}^{m-1}\pbinom(k;m-1,p)\sum_{|S|=k+1} \E_{S}^T\ys = \z\right\},\]
where $\E_S = \I[S,]/{m -1\choose |S|-1}$ and $\I\in\mathbb R^{m\times m}$ is the identity matrix.
\end{corollary}
\begin{proof}
This is simply a direct application of Lemma \ref{lemma:dualNorm}.
\end{proof}


\begin{corollary}
\label{col:SOCP}
The dual norms $\vertiii{.}_s^*$ and $\vertiii{.}_p^*$ can be computed via a Second Order Cone Program (SOCP).
\end{corollary}
\begin{proof}
Introducing additional variable $t\geq 0$, the problem of computing $\vertiii{\z}_s^*$ is equivalent to the following formulation
\begin{align}
\label{SOCP}
\inf_{t,\ys}  \hspace{2mm} t \hspace{2mm}
& s.t.  \ \|\ys\|_2 \leq t \text{ for all $S$ such that $|S|=s$} \nonumber \\
& \text{and } \sum_{|S|=s} \E_S^T\ys = \z. \nonumber
\end{align}
Notice that the above program is already in the standard form of SOCP. The case of $\vertiii{.}_p^*$ can be handled in a similar manner. 
\end{proof}

\subsection{Inequalities of $\vertiii{.}_s$ and $\vertiii{.}_p$ and their duals}
As demonstrated in the last section, it is in general expensive to compute $\vertiii{.}_s^*$ and $\vertiii{.}_p^*$. In this section, we will derive sharp and easy-to-compute lower and upper bounds to approximate these quantities.
\begin{lemma} (Monotonicity of $\vertiii{\z}_s$ and $\vertiii{\z}_p$)
\label{lemma:normDecreasing}
Let $\z\in \mathbb R^m$. $\vertiii{\z}_1 = \|\z\|_1$ and $\vertiii{\z}_m = \|\z\|_2$. For $1\leq l< k \leq m$, we have $\vertiii{\z}_l\geq \vertiii{\z}_k$; similarly for $0<p<q<1$, $\vertiii{\z}_p\geq \vertiii{\z}_q$. Furthermore, the equalities hold iff the vector $\z$ contains at most one non-zero entry.
\end{lemma}
\begin{proof} By definition, we have
\[
\vertiii{\w}_1 = \frac{\sum_{|S|=1}\|\w[S]\|_2}{{m-1 \choose 1-1}} = \|\w\|_1.
\]
Similarly,
\[
\vertiii{\w}_m = \frac{\sum_{|S|=m}\|\w[S]\|_2}{{m-1 \choose m-1}} = \|\w\|_2.
\]
For $1\leq k\leq m-1$, let $S'$ be a subset of $\intSet{m}$ such that $|S'|=k+1$. By triangle inequality 
\[\sum_{|S|=k,S\subset S'}\|\z[S]\|_2\geq k \|\z[S']\|_2,\] 
where the equality holds iff $\|\z[S']\|_0\leq 1$. Thus 
\[\sum_{|S'|=k+1}\sum_{|S|=k,S\subset S'}\|\z[S]\|_2\geq k \sum_{|S'|=k+1}\|\z[S']\|_2,\] 
and the equality holds iff $\|\z\|_0\leq 1$. Notice that the LHS of the above inequality is simply $(m-k)\sum_{|S|=k}\|\z[S]\|_2$. Therefore 
\[\vertiii{\z}_{k} = {m-1\choose k-1}^{-1}\sum_{|S|=k}\|\z[S]\|_2 \geq {m-1\choose k}^{-1}\sum_{|S|=k+1}\|\z[S]\|_2=\vertiii{\z}_{k+1},\] 
and so the inequality holds. 

For $\vertiii{.}_p$, let $Y$ be a random variable that follows the binomial distribution with parameters $m-1$ and $p$. Observe that $\vertiii{\z}_p = \mathbb E \vertiii{\z}_{Y+1}$, where the expectation is taken with respect to $Y$.  If $\|\z\|_0> 1$, $\vertiii{\z}_k$ is strictly decreasing in $k$ by the first part. Hence, $\vertiii{\z}_p$ as a function of $p$ is also strictly decreasing on $(0,1)$. Indeed, it can be shown that
\[
\frac{d}{dp}\vertiii{\z}_p
= \sum_{k=0}^{m-1}\pbinom(k;K-1,p)\left(\vertiii{\z}_{k+1} - \vertiii{\z}_k\right)<0.
\]
If $\|\z\|_0\leq 1$, then $\vertiii{\z}_1 = \vertiii{\z}_m$ and so $\frac{d}{dp}\vertiii{\z}_p = 0$. Therefore $\vertiii{\z}_p = \vertiii{\z}_1$ is a constant in $p$.  On the other hand, if $\vertiii{\z}_p = \vertiii{\z}_q$ for $0<p<q<1$, by the fact that $\frac{d}{dp}\vertiii{\z}_p\leq 0$, we must have $\frac{d}{dp}\vertiii{\z}_p = 0$ and so $\vertiii{\z}_{k-1} = \vertiii{\z}_k$ for all $k\in \intSet{m}$. Thus $\|\z\|_0\leq 1$.

\end{proof}

\begin{corollary} (Monotonicity of $\vertiii{\z}_s^*$ and $\vertiii{\z}_p^*$)
\label{lemma:normIncreasing}
Let $\z\in \mathbb R^m$. $\vertiii{\z}_1^* = \|\z\|_\infty$ and $\vertiii{\z}_m^* = \|\z\|_2$. For $1\leq i < j \leq m$, we have $\vertiii{\z}_i^*\leq \vertiii{\z}_j^*$; similarly for $0<p<q<1$, $\vertiii{\z}_p^*\leq \vertiii{\z}_q^*$. Furthermore, the equalities hold iff the vector $\z$ contains at most one non-zero entry.
\end{corollary}
\begin{proof}
This is a direct consequence of Lemma \ref{lemma:normDecreasing} and the dual norm definition $\vertiii{\z}^* = \sup_{\mathbf{y}\ne 0}\frac{\z^T\mathbf{y}}{\vertiii{\mathbf{y}}}$.  
\end{proof}

\begin{lemma} 
\label{lemma:normRelation}
Let  $p\in(0,1)$ and $k = \lceil (m-1)p+1 \rceil$. For any $\z \in \mathbb R^{m}$, we have
\begin{enumerate}
\item $\vertiii{\z}_p \geq \vertiii{\z}_k$.
\item $\vertiii{\z}_p^* \leq \vertiii{\z}_k^*$.
\end{enumerate}
\end{lemma}
\begin{proof}
Define the function $f$ with domain on $[1,m]$ as follows: let $f(1) = \vertiii{\z}_{1} = \|\z\|_1$; for $i\in\intSet{m-1}$ and $a\in (i,i+1]$, define
\[f(a) = \vertiii{\z}_{i} + (\vertiii{\z}_{i+1} - \vertiii{\z}_{i})(a-i).\]
It is clear that $f$ is piecewise linear by construction. In addition, by Lemma \ref{lemma:normConvex}, $f$ is also convex. Notice that $\vertiii{\z}_p = \mathbb E \vertiii{\z}_{Y+1} = \mathbb E f(Y+1)$, where $Y$ is a random variable from the binomial distribution with parameters $m-1$ and $p$. By Jensen's inequality, 
\[\mathbb E f(Y+1) \geq f(\mathbb EY + 1) =f((m-1)p+1).\]
Thus by Lemma \ref{lemma:normDecreasing}, $\vertiii{\z}_p \geq \vertiii{\z}_k$ for all $k\geq (m-1)p+1$. So the first part follows.

To upperbound $\vertiii{\z}_p^*$, notice that if $k\geq (m-1)p+1$, 
\[
\vertiii{\z}_p^* = \sup_{\w\ne 0}\frac{\w^T\z}{\vertiii{\w}_p}
\leq \sup_{\w\ne 0}\frac{\w^T\z}{\vertiii{\w}_k} = 
\vertiii{\z}_k^*.
\]
\end{proof}

For the following lemmas, the quantities $\hb{m}{d}{a}$ and $L_{m}(d,k)$ are defined as in Definition \ref{def:tau}. 
\begin{lemma} (Approximating $\hb{m}{d}{a}$)
\label{lemma:hb}
For $d\in \intSet{m}$ and $a\in (0,m]$:
\[
 \hb{m}{d}{a} \leq \sqrt{\frac{da}{m}}. 
\]
\end{lemma}
\begin{proof}
For $k\in \intSet{m}$, by Jensen's inequality,
\[
\mathbb E \sqrt{L_{m}(d,k)} \leq \sqrt{ \mathbb E L_{m}(d,k)} = \sqrt{\frac{dk}{m}}.
\]
Note that the last equality follows from the expectation of a hypergeometric random variable. Now suppose $a\in (k-1,k]$. By the above inequality and apply Jensen's inequality one more time, we have
\begin{align*}
\hb{m}{d}{a} & =  (k - a ) \mathbb E \sqrt{L_{m}(d, k -1)} + (1-(k - a ))\mathbb E \sqrt{L_{m}(d,k)} \\
& \leq (k - a ) \sqrt{\frac{d(k -1)}{m}} + (1-(k - a )) \sqrt{\frac{dk}{m}} = \sqrt{\frac{da}{m}}.
\end{align*}
\end{proof}

\begin{lemma}
\label{lemma:lower} (Lower bounds for $\vertiii{\z}_s^*$ and $\vertiii{\z}_p^*$)
Let $\z \in \mathbb R^m$. 
We have 
\begin{enumerate}
\item For $s\in\intSet{m}$, 
\[\vertiii{\z}_s^* \geq \frac{s}{m}\max_{T\subset\intSet{m}}\frac{\|\z[T]\|_1}{\hb{m}{|T|}{s}}
\geq \max\left( \|\z\|_\infty, \sqrt{\frac{s}{m}}\max_{T\subset\intSet{m}}\frac{\|\z[T]\|_1}{\sqrt{|T|}}\right) .\]
\item For $p\in(0,1)$,
\begin{align*}
\vertiii{\z}_p^* & \geq p \max_{T\subset\intSet{m}} \left\{ \Big(\sum_{k=0}^{m}\pbinom(k,m,p)\hb{m}{|T|}{k}\Big)^{-1}\|\z[T]\|_1 \right\} \\
& \geq p  \max_{T\subset\intSet{m}}\frac{\|\z[T]\|_1}{\hb{m}{|T|}{pm}} = \max \left( \|\z\|_\infty, \sqrt{p}\max_{T\subset\intSet{m}}\frac{\|\z[T]\|_1}{\sqrt{|T|}} \right).
\end{align*}
\end{enumerate}
\end{lemma}

\begin{proof}
(1) Note that by definition,
\[
\vertiii{\z}_s^* = \sup_\w \frac{\z^T\w}{\vertiii{\w}_s}
\]
Let $d \in \intSet{m}$ and $T\subset\intSet{m}$ such that $|T|=d$. Define $\w \in \mathbb R^m$ such that $\w[i] = 1$ for $i\in T$ and $\w[i] = 0$ for $i\in T^c$. We have:
\begin{align*}
\vertiii{\w}_s
= {m-1 \choose s-1}^{-1} \sum_{|S|=s} \|\w[S]\|_2 
& = {m-1 \choose s-1}^{-1} \sum_{l=\max(0,s+d-m)}^{\min(s,d)}\sum_{|S|=s, |S\cap T| = l } \|\w[S]\|_2 \\
& = {m-1 \choose s-1}^{-1} \sum_{l=\max(0,s+d-m)}^{\min(s,d)}\sum_{|S|=s, |S\cap T| = l } \sqrt{l} \\
& = {m-1 \choose s-1}^{-1} \sum_{l=\max(0,s+d-m)}^{\min(s,d)} {d \choose l}{m-d \choose s - l}\sqrt{l} \\
& = \frac{m}{s} \mathbb E \sqrt{L_{m}(d,s)} = \frac{m}{s} \hb{m}{d}{s}.
\end{align*}
Thus for all $d \in \intSet{m}$ and any subset $T$ such that $|T|=d$, we have shown
\[
\vertiii{\z}_s^*\geq \frac{s}{m}\frac{\|\z[T]\|_1}{\hb{m}{d}{s}}.
\] 
Note that if $d= 1$, $\mathbb E\sqrt{L_{m}(d,s)} = \frac{s}{m}$. Therefore
\[ \frac{s}{m}\frac{\|\z[T]\|_1}{\hb{m}{d}{s}} \geq \|\z\|_\infty,\]
Moreover, by Lemma \ref{lemma:hb},
\[
\hb{m}{d}{s} \leq \sqrt{\frac{ds}{m}}.
\]
Hence we have
\[ \frac{s}{m}\frac{\|\z[T]\|_1}{\hb{m}{d}{s}} \geq \sqrt{\frac{s}{m}} \frac{\|\z[T]\|_1}{\sqrt{d}} ,\]
and the first part of the claim follows.

\noindent (2) For the same $\w \in \mathbb R^m$ defined previously, 
\begin{align*}
\vertiii{\w}_p & =
\sum_{k=0}^{m-1} \pbinom(k,m-1,p)\vertiii{\w}_{k+1} \\
& = m \sum_{k=0}^{m-1} \pbinom(k,m-1,p) \frac{\hb{m}{d}{k+1}}{k+1} \\ 
& =  m \sum_{k=0}^{m-1}{m-1 \choose k}p^k(1-p)^{m-k-1}\frac{1}{k+1}\hb{m}{d}{k+1} \\
& = \frac{1}{p}\sum_{k=0}^{m-1}{m \choose k+1}p^{k+1}(1-p)^{m-(k+1)}\hb{m}{d}{k+1} \\
& = \frac{1}{p}\sum_{k=0}^{m}{m \choose k}p^k(1-p)^{m-k}\hb{m}{d}{k}.
\end{align*}
Thus for all $d \in \intSet{m}$ and any subset $T$ such that $|T|=d$, we have shown
\[
\vertiii{\z}_p^* \geq p \Big( \sum_{k=0}^{m}\pbinom(k,m,p)\hb{m}{d}{k} \Big)^{-1}\|\z[T]\|_1.
\]
Next, we will show
\[
\sum_{k=0}^{m}\pbinom(k,m,p)\hb{m}{d}{k} \leq \hb{m}{d}{pm}.
\]
To this end, let us first notice that the LHS quantity is a binomial average of $\hb{m}{d}{k}$ with respect to $k$. By construction, $\hb{m}{d}{.}$ is piecewise linear. Furthermore, $\hb{m}{d}{.}$ is also concave by Lemma \ref{hypergeom}. Now let $Y$ be a random variable having the binomial distribution with parameters $m$ and $p$.  By Jensen's inequality,
\[
\sum_{k=0}^{m}\pbinom(k,m,p)\hb{m}{d}{k} = \mathbb E \hb{m}{d}{Y} \leq \hb{m}{d}{\mathbb E Y} = \hb{m}{d}{mp}. 
\]
In particular, if $d=1$, it is easy to see that $\hb{m}{d}{mp} = p$. So
\[
p \left( \max_{T\subset \intSet{m},|T|=1} \Big( \sum_{k=0}^{m}\pbinom(k,m,p)\hb{m}{|T|}{k} \Big)^{-1}\|\z[T]\|_1 \right)\geq \|\z\|_\infty.
\]
On the other hand, by Lemma \ref{lemma:hb},
\begin{align*}
\hb{m}{d}{pm} 
 \leq \sqrt{\frac{d}{m}}\sqrt{pm} = \sqrt{pd}.
\end{align*}
Therefore
\[
p \Big( \sum_{k=0}^{m}\pbinom(k,m,p)\hb{m}{d}{k} \Big)^{-1}\|\z[T]\|_1 \geq \sqrt{p}\frac{\|\z[T]\|_1}{\sqrt{d}},
\]
and the proof is complete.
\end{proof}

\begin{lemma} (Upper bounds for $\vertiii{\z}_s^*$ and $\vertiii{\z}_p^*$)
\label{lemma:upper}
Let $\z \in \mathbb R^m$. 
\begin{enumerate}
\item For $s\in\intSet{m}$, 
\[
\vertiii{\z}_s^* \leq \max_{|S|=s}\|\z[S]\|_2.
\]
\item For $p\in(0,1)$,
\[
\vertiii{\z}_p^*\leq \max_{|S|=k}\|\z[S]\|_2,
\] 
where $k = \lceil p(m-1)+1 \rceil$.
\end{enumerate}
\end{lemma}
\begin{proof}
To establish the upper bound, we will use the equivalent formulation of $\vertiii{.}_s^*$ in Corollary \ref{col:dualNorm}. For $S\subset \intSet{m}$ of size $s$, as in Corollary \ref{col:dualNorm}, let $\E_S = \I[S,]/{m -1\choose s-1}$ where $\I\in\mathbb R^{m\times m}$ is the identity matrix. If we set $\ys = \z[S]$, then $\sum_{|S|=s}\E_S^T\ys = \z$ and so $\{\ys\}$ is feasible. Therefore
\[
\vertiii{\z}_s^* \leq \max_{|S|=s}\|\z[S]\|_2. 
\]
The upperbound of $\vertiii{\z}_p^*$ follows from the inequality $\vertiii{\z}_p^*\leq \vertiii{\z}_k^*$ for $k = \lceil p(m-1)+1 \rceil$ by the second part of Lemma \ref{lemma:normRelation}. 
\end{proof}

\begin{corollary} ($1$-sparse vectors)
\label{thm:oneMu}
Let $\z = (z,0,...,0)^T \in \mathbb R^m$. We have 
\[
\vertiii{\z}_s^* = \vertiii{\z}_p^* = |z|.
\]
\end{corollary}
\begin{proof}
These are direct consequences of Lemma \ref{lemma:lower} and Lemma \ref{lemma:upper}.
\end{proof}

\begin{corollary} (All-constant vectors)
\label{thm:constMu}
Let $\z \in \mathbb R^m$ be such that $\z[i] = z$ for all $i\in\intSet{m}$. We have 
\begin{enumerate}
\item $\vertiii{\z}_s^* = \sqrt{s}|z|$.
\item $\vertiii{\z}_p^* = mp\Big(\sum_{k=0}^{m}\pbinom(k,m,p)\sqrt{k}\Big)^{-1}|z|$.
\end{enumerate}
\end{corollary}
\begin{proof}
First of all, note that $L(m,k) = k$ and $\mathbb E \sqrt{L(m,k)} = \sqrt{k}$. Thus by Lemma \ref{lemma:lower} and \ref{lemma:upper}, we have
\[
\vertiii{\z}_s^* = \sqrt{s}|z|. 
\]
So the first part of the claim is verified. Next, by Lemma \ref{lemma:lower},
\[
\vertiii{\z}_p^* \geq mp\Big(\sum_{k=0}^{m}\pbinom(k,m,p)\sqrt{k}\Big)^{-1}|z|.
\] 
On the other hand, for $S$ such that $|S|=s$, we can define
 \[
 \ys = \frac{mp}{\sqrt{s}}\Big(\sum_{k=0}^{m}\pbinom(k,m,p)\sqrt{k}\Big)^{-1} (z,...,z)^T\in\mathbb R^s,
 \]
For notation simplicity, let $c = \frac{1}{mp}\Big(\sum_{k=0}^{m}\pbinom(k,m,p)\sqrt{k}\Big)$. As in Corollary \ref{col:dualNorm}, for $S\subset \intSet{m}$, let $\E_S = \I[S,]/{m -1\choose |S|-1}$. For $i\in\intSet{m}$, we have
 \begin{align*}
 \sum_{k=0}^{m-1}\pbinom(k;m-1,p)\sum_{|S|=k+1} (\E_{S}^T\ys)[i]
 & =   c^{-1}\sum_{k=0}^{m-1}\pbinom(k;m-1,p)\frac{1}{\sqrt{k+1}} \\
 & =  c^{-1}\frac{z}{mp}\sum_{k=0}^{m}\pbinom(k;m,p)\sqrt{k} = z.
 \end{align*}
 Thus by Corollary \ref{col:dualNorm}, 
 \[
 \vertiii{\z}_p^*\leq \max_{S}\|\ys\|_2 =
  mp \Big(\sum_{k=0}^{m}\pbinom(k,m,p)\sqrt{k}\Big)^{-1} |z|,
 \]
 and the proof is complete.
\end{proof}

\begin{lemma} (Convexity of $\vertiii{\z}_k$)
\label{lemma:normConvex}
Let $\z\in \mathbb R^m$, where $m\geq3$. For $k\in\intSet{m-2}$, we have the following inequality
\begin{align}
\label{eq:unnorm}
\vertiii{\z}_k + \vertiii{\z}_{k+2} \geq 2\vertiii{\z}_{k+1}.
\end{align}
\end{lemma}
\begin{proof}
We will first show that the claim is true for $k = m-2$. Notice that in this case $\vertiii{\z}_{k+2} = \vertiii{\z}_m = \|\z\|_2$. If $\|\z\|_2 = 0$, the inequality (\ref{eq:unnorm}) is trivially true. Now suppose $\|\z\|_2>0$, dividing both sides of the inequality by $\|\z\|_2$, we have
\[
{m-1 \choose m-3}^{-1} \sum_{|S|=m-2} \frac{\|\z[S]\|_2}{\|\z\|_2}  + 1 \geq 2 {m-1 \choose m-2}^{-1} \sum_{|S|=m-1} \frac{\|\z[S]\|_2}{\|\z\|_2}.
\]
Now let $\x = (x_1,...,x_m)^T \in\mathbb R^m$ be such that $x_i = \z[i]^2/\|\z\|_2^2$. It suffices to show
\begin{align}
\label{eq:norm}
\sum_{|S|=m-2} \left(\sum_{i\in S} x_i\right)^{1/2}  + \frac{(m-1)(m-2)}{2} \geq (m-2) \sum_{i=1}^m \sqrt{ 1 -  x_i},
\end{align}
for all $\x\geq 0$ entry-wise such that $\sum_i x_i = 1$. We will now prove the above inequality by induction on $m$. First of all, notice that for the base case where $m=3$, we need to show: 
\[\sqrt{x_1}+\sqrt{x_2}+\sqrt{x_3} + 1 \geq \sqrt{1 - x_1}+\sqrt{1-x_2}+\sqrt{1-x_3},\]
with the constraints $x_i\geq0$ and $x_1+x_2+x_3 = 1$. For fixed $x_3$, let 
\[
f(x_1) = \sqrt{x_1}+\sqrt{1-x_1-x_3}+\sqrt{x_3} + 1 - \sqrt{x_1+x_3}-\sqrt{1-x_1}-\sqrt{1-x_3}.
\]
We will show that $f(x_1)$ is minimized at $x_1 =0$ or $x_1 = 1-x_3$. Suppose now $x_1>0$. Taking derivative with respect to $x_1$:
\[
f'(x_1) = \frac{1}{2}(\frac{1}{\sqrt{x_1}} - \frac{1}{\sqrt{1- x_1 -x_3}}-\frac{1}{\sqrt{x_1 + x_3}} + \frac{1}{\sqrt{1- x_1}}).   
\]
Let $l(x_1) = \frac{1}{\sqrt{x_1}} - \frac{1}{\sqrt{x_1 + x_3}}$. Note that $f'(x_1) = \frac{1}{2}l(x_1) - \frac{1}{2}l(1-x_3-x_1)$. Now we have
\[
l'(x_1) = \frac{1}{2}(x_1+x_3)^{-3/2}-\frac{1}{2}{x_1}^{-3/2}.
\]
So $l(x_1)$ is decreasing on $(0,1-x_3)$ and by symmetry the function $l(1-x_3-x_1)$ is increasing on $(0,1-x_3)$. On the other hand, since $\lim_{x_1\downarrow 0^+} l(x_1) = +\infty$ and $\lim_{x_1\downarrow 0^+} l(1-x_3-x_1) = -\infty$, we know that $f'(x_1)>0$ on $(0,\frac{1-x_3}{2})$ and $<0$ on $(\frac{1-x_3}{2},1-x_3)$. Thus, the minimum of $f$ can only be attained at the boundaries, i.e. $x_1 = 0$ or $x_1 = 1-x_3$. In either case we have
\begin{align*}
& \sqrt{x_1}+\sqrt{x_2}+\sqrt{x_3} + 1 - \sqrt{1 - x_1} - \sqrt{1-x_2} - \sqrt{1-x_3}\\
\geq & \sqrt{x_2}+\sqrt{x_3} - \sqrt{1-x_2} - \sqrt{1-x_3}
= 0,
\end{align*}
as $x_2 + x_3 = 1$. So we establish (\ref{eq:norm}) for $m=3$.

Suppose (\ref{eq:norm}) is also true for $m = n-1$. For $m = n$, similar to the $m=3$ case, for fixed $x_3,...,x_n$, define 
\[
f(x_1) = \sum_{|S|=n-2} \bigg(\sum_{i\in S} x_i\bigg)^{1/2}  + \frac{(n-1)(n-2)}{2} - (n-2) \sum_{i=1}^n \sqrt{ 1 -  x_i},
\] 
subject to $x_i\geq0$ and $\sum_i x_i = 1$. Again, we will show $f$ attains its minimum at either $x_1 = 0$ or $x_1 = 1-\sum_{i=3}^nx_i$. Notice that 
\begin{align*}
 \sum_{|S|=n-2} \Big(\sum_{j\in S} x_j\Big)^{1/2}
= & \sum_{|S|=n-3, 1,2\not\in S} \Big(x_1+\sum_{j\in S} x_j\Big)^{1/2}
+ \sum_{|S|=n-4, 1,2\not\in S} \Big(x_1+x_2+\sum_{j\in S} x_j\Big)^{1/2} \\
& + \sum_{|S|=n-3, 1,2\not\in S} \Big(x_2+\sum_{j\in S} x_j\Big)^{1/2}
+ \Big(\sum_{j=3}^n x_j\Big)^{1/2} \\
= & \sum_{i=3}^n \Big(x_1+\sum_{j=3}^n x_j -x_i\Big)^{1/2}
+ \sum_{3\leq i < j \leq n} (1-x_i-x_j)^{1/2} \\
& +\sum_{i=3}^n (1-x_1-x_i)^{1/2}
+\Big(\sum_{j=3}^n x_j\Big)^{1/2}.
\end{align*}
In addition,
\begin{align*}
\sum_{i=1}^n ( 1 -  x_i)^{1/2}
= (1-x_1)^{1/2}+\Big(x_1+\sum_{j=3}^n x_j\Big)^{1/2} + \sum_{i=3}^n(1-x_i)^{1/2}.
\end{align*}
Taking derivative with respect to $x_1$,
\begin{align*}
f'(x_1) = & \frac{1}{2}\bigg(\sum_{i=3}^n\Big(x_1+\sum_{j=3}^nx_j-x_i\Big)^{-1/2} - \sum_{i=3}^n(1-x_1-x_i)^{-1/2} \\
& + (n-2)(1-x_1)^{-1/2}-(n-2)\Big(x_1+\sum_{i=3}^n x_i\Big)^{-1/2}\bigg).
\end{align*}
Now let
\[
l(x_1) = \sum_{i=3}^n\Big(x_1+\sum_{j=3}^nx_j-x_i\Big)^{-1/2} - (n-2)\Big(x_1+\sum_{j=3}^n x_j\Big)^{-1/2}.
\]
So $2f'(x_1) = l(x_1) - l(1-\sum_{i=3}^nx_i -x_1)$. Again 
\[
l'(x_1) = -\frac{1}{2}\sum_{i=3}^n \Big(x_1+\sum_{j=3}^nx_j-x_i\Big)^{-3/2} + \frac{n-2}{2}\Big(x_1+\sum_{j=3}^n x_j\Big)^{-3/2}.
\]
It is easy to see that $l'(x_1)<0$ and so $l(x_1)$ is decreasing on $(0,1-\sum_{i=3}^nx_i -x_1)$. On the other hand $\lim_{x_1\downarrow 0^+}l(x_1) = +\infty$. By symmetry $f'(x_1)>0$ on $\big(0,\frac{1}{2}(1-\sum_{i=3}^nx_i -x_1)\big)$ and $<0$  on $\big(0,\frac{1}{2}(1-\sum_{i=3}^nx_i -x_1)\big)$. So $f$ attains its minimum at $x_1 = 0$ or $x_1 = 1 - \sum_{i=3}^n x_i$. Hence we have
\begin{align}
\label{eq:lowerBound}
& \sum_{|S|=n-2} \Big(\sum_{i\in S} x_i\Big)^{1/2}  + \frac{(n-1)(n-2)}{2} - (n-2) \sum_{i=1}^n ( 1 -  x_i)^{1/2} \nonumber \\
\geq & \Big(\sum_{|S|=n-3,1\not \in S} + \sum_{|S|=n-2,1\not\in S}\Big)\Big(\sum_{j\in S} x_j\Big)^{1/2}
+\frac{(n-2)(n-3)}{2} - (n-2)\sum_{i=2}^n (1-x_i)^{1/2}.
\end{align}
By the induction assumption that (\ref{eq:norm}) holds when $m=n-1$, we have
\begin{align*}
\sum_{|S|=n-3,1\not \in S}\Big(\sum_{j\in S} x_j\Big)^{1/2} +\frac{(n-2)(n-3)}{2} \geq (n-3)\sum_{i=2}^n(1-x_i)^{1/2}.
\end{align*}
Thus (\ref{eq:lowerBound}) is greater than or equal to
\begin{align*}
& - \frac{(n-2)(n-3)}{2} + \sum_{|S|=n-2,1\not\in S}\Big(\sum_{j\in S} x_j\Big)^{1/2} +\frac{(n-1)(n-2)}{2} - (n-2) - \sum_{i=2}^n (1-x_i)^{1/2} \\
= & \sum_{|S|=n-2,1\not\in S} \Big(\sum_{j\in S} x_j\Big)^{1/2} - \sum_{i=2}^n (1-x_i)^{1/2}
= \sum_{i=2}^n (1- x_i)^{1/2} - \sum_{i=2}^n (1-x_i)^{1/2} =0.
\end{align*}
Thus we have verified the claim that (\ref{eq:norm}) and hence (\ref{eq:unnorm}) holds for $k = m-2$ for all $m\geq3$. To establish the case for general $1\leq k \leq m-2$, we again perform induction on the $(m,k)$-tuple. Note that the base case $m=3$ and $k=1$ has been previously proved. Suppose (\ref{eq:unnorm}) holds for $m = n-1$ and $1\leq k \leq n-3$. Now consider $m=n$ and $1\leq k <n-2$. Notice that
\begin{align*}
\vertiii{\z}_k = & \frac{1}{n-k}{n-1\choose k-1}^{-1}\sum_{|T|=n-1}\sum_{|S|=k,S\subset T}\|\z[S]\|_2 \\
= & (n-1) {n-2\choose k-1}^{-1} \sum_{|T|=n-1}\sum_{|S|=k,S\subset T}\|\z[S]\|_2 \\
= & (n-1) \sum_{|T|=n-1} \vertiii{\z[T]}_k.
\end{align*}
By the induction assumption, for all $T$ such that $|T|=n-1$, we have:
\[
\vertiii{\z[T]}_k + \vertiii{\z[T]}_{k+2} \geq 2 \vertiii{\z[T]}_{k+1}.
\]
Therefore
\begin{align*}
\vertiii{\z}_k + \vertiii{\z}_{k+2} - 2 \vertiii{\z}_{k+1}
=  (n-1) \sum_{|T|=n-1}(\vertiii{\z[T]}_k + \vertiii{\z[T]}_{k+2} - 2 \vertiii{\z[T]}_{k+1})
\geq 0.
\end{align*}
Thus the claim also holds for $m=n$ and $1\leq k < n - 2$, completing the proof.
\end{proof}

\subsection{Miscellaneous}
\begin{lemma} (Concavity of $\mathbb E \sqrt{L_{m}(d,k)}$) Let $d\in\intSet{m}$. For $k\in\intSet{m-2}$, we have
\label{hypergeom}
\begin{align}
\label{eqn:concave}
\mathbb E \sqrt{L_{m}(d,k)} + \mathbb E \sqrt{L_{m}(d,k+2)} \leq 2 \mathbb E \sqrt{L_{m}(d,k+1)}.
\end{align}
where the geometric random variable $L_{m}(d,k)$ is defined as in Definition \ref{def:tau}.
\end{lemma}
\begin{proof}
Suppose we are now sampling without replacement from a pool of numbers with $d$ 1's and $m-d$ 0's. For $i\in\intSet{m}$, denote by $X_i \in \{0,1\}$ the $i$-th outcome. It is easy to see that $L_{m}(d,k)$ and $\sum_{i=1}^k X_i$ have the same distribution. To show (\ref{eqn:concave}), it suffices to prove the following conditional expectation inequality:
\[
\sqrt{L_{m}(d,k)} + \mathbb E [\sqrt{L_{m}(d,k+2)} \ | \ L_{m}(d,k)] \leq 2 \mathbb E [\sqrt{L_{m}(d,k+1)} \ | \ L_{m}(d,k)]
\]
Note that the above inequality follows if for all $0\leq a \leq \min(d,k)$:
\[
\sqrt{a} + \mathbb E \sqrt{a + X_{k+1} + X_{k+2}}  \leq 2 \mathbb E \sqrt{a + X_{k+1}} 
\]
It is easy to see that
\begin{align*}
\mathbb E \sqrt{a + X_{k+1}}  & = \frac{d - a}{m - k}\sqrt{a+1} + (1-\frac{d - a}{m - k})\sqrt{a}. \\
\mathbb E \sqrt{a + X_{k+1} + X_{k+2}} &  = 
  \frac{d - a}{m - k}\times\frac{d-a-1}{m-k-1}\sqrt{a+2} + 2\times\frac{d-a}{m-k}\times\frac{m-k-(d-a)}{m-k-1}\sqrt{a+1} \\
& \ \ + \frac{m -k - (d-a)}{m-k}\times\frac{m-k-(d-a)-1}{m-k-1}\sqrt{a}.
\end{align*}
By elementary algebra, it can be shown that
\begin{align*}
& 2 \mathbb E \sqrt{a + X_{k+1}}  - \sqrt{a} - \mathbb E \sqrt{a + X_{k+1} + X_{k+2}} \\
= &  \frac{d - a}{m - k}\times\frac{d-a-1}{m-k-1}\times(2\sqrt{a+1} - \sqrt{a+2}-\sqrt{a}) \geq 0,
\end{align*}
The inequality follows since $f(x) = \sqrt{x}$ is a concave function. Thus the proof is complete.
\end{proof}

\begin{lemma}
\label{lemma:diffNorm1}
Let $\x(t) = (x_1(t),...,x_m(t))^T\in \mathbb R^m$ be an $m$-dimensional function on $[0,\epsilon)$ such that: (1) $x_1(0)=1$ and for all $i\geq 2$, $x_i(0)=0$; (2) The derivative $\dot{x}_i(t)$ exists and is bounded for all $t \in (0,\epsilon)$. We have
\[
\lim_{t\downarrow 0^+}\frac{\|\x(t)\|_2 - \|\x(0)\|_2}{t} = \lim_{t\downarrow 0^+} \dot{\x}_1(t).
\]
\end{lemma}
\begin{proof}

\begin{align*}
\lim_{t\downarrow 0^+} \frac{\|\x(t)\|_2 - \|\x(0)\|_2}{t} & = \lim_{t\downarrow 0^+} \frac{(\sum_{i=1}^m x_i^2(t))^{1/2} - 1}{t} \\
& = \lim_{t\downarrow 0^+} \frac{\sum_{i=1}^m x_i^2(t) - 1}{t}((\sum_{i=1}^m x_i^2(t))^{1/2} + 1)^{-1} \\
& = \frac{1}{2}\lim_{t\downarrow 0^+} \frac{\sum_{i=1}^m x_i^2(t) - 1}{t} \\
& = \frac{1}{2}(\lim_{t\downarrow 0^+} \frac{x_1^2(t) - 1}{t} + \sum_{i=2}^m \lim_{t\downarrow 0^+} \frac{x_i^2(t)}{t}) \\
& = \frac{1}{2}(\lim_{t\downarrow 0^+} \frac{x_1(t) - 1}{t}(x_1(t) +1) + \sum_{i=2}^m \lim_{t\downarrow 0^+} \frac{x_i^2(t)}{t}) \\
& = \lim_{t\downarrow 0^+} \frac{x_1(t) - 1}{t} + \frac{1}{2}\sum_{i=2}^m \lim_{t\downarrow 0^+} \frac{x_i^2(t)}{t}.
\end{align*}
By mean value theorem, for each $t\in (0,\epsilon)$, there exists $\delta_t\in(0,t)$ such that $x_1(t) - 1 = \dot{x}_1(\delta_t)t$. Thus the first term simply becomes $\lim_{t\downarrow 0^+} \dot{x}_1(t)$. By the same argument, for each $i \in \{2,...,m\}$, $x_i(t) = \dot{x}_i(\delta_t)t$ for some $\delta_t\in(0,t)$. Since $\dot{x}_i(t)$ is bounded, we have
\[
 \lim_{t\downarrow 0^+} \frac{x_i^2(t)}{t} =  \lim_{t\downarrow 0^+} \dot{x}_i(\delta_t)^2t = 0.
\] 
Therefore the claim is verified.
\end{proof}

\begin{lemma}
\label{lemma:diffNorm}
Let $\x(t) = (x_1(t),...,x_m(t))^T\in \mathbb R^m$ be an $m$-dimensional function on $[0,\epsilon)$ such that: (1) $x_i(0)=0$ for all $i = 1,...,m$; (2) The derivative $\dot{x}_i(t)$ exists for all $t \in (0,\epsilon)$. We have
\[
\lim_{t\downarrow 0^+}\frac{\|\x(t)\|_2}{t} = \|\lim_{t\downarrow 0^+} \dot{\x}(t)\|_2.
\]
\end{lemma}
\begin{proof}

\begin{align*}
\lim_{t\downarrow 0^+}\frac{\|\x(t)\|_2}{t} = \lim_{t\downarrow0^+}(\sum_{i=1}^m (\frac{x_i(t)}{t})^2)^{1/2} 
 = (\sum_{i=1}^m (\lim_{t\downarrow0^+}\frac{x_i(t)}{t})^2)^{1/2} = \|\lim_{t\downarrow 0^+} \dot{\x}(t)\|_2.
\end{align*}

\end{proof}

\begin{lemma}
\label{lemma:diffNorm3}
Let $\abf = (a_1,...,a_m)^T\in\mathbb R^m$ where $a_1 \neq 0$ and $\x(t) = (x_1(t),...,x_m(t))^T\in \mathbb R^m$ be an $m$-dimensional function on $[0,\epsilon)$ such that: (1) $x_1(0)=1$ and for all $i\geq 2$, $x_i(0)=0$; (2) The derivative $\dot{x}_i(t)$ exists and is bounded for all $t \in (0,\epsilon)$. We have
\begin{align*}
\lim_{t\downarrow 0^+}\frac{|\abf^T\x(t)| - |a_1|}{t} = |a_1|\lim_{t\downarrow 0^+} \dot{x}_1(t) + \sgn(a_1)\sum_{i=2}^m a_i \lim_{t\downarrow 0^+} \dot{x}_i(t).
\end{align*}
\end{lemma}
\begin{proof} Without loss of generality, assume $a_1>0$. Since $x_1(0)=1$ and for all $i\geq 2$, $x_i(0)=0$, by continuity, for sufficiently small $t$, we have
\begin{align*}
\frac{|\abf^T\x(t)| - |a_1|}{t} = \frac{|a_1x_1(t) + \sum_{i=2}^m a_ix_i(t)| - a_1}{t} = \frac{a_1x_1(t) - a_1 + \sum_{i=2}^m a_ix_i(t)}{t}.
\end{align*}
Therefore, by the same argument in the proof of Lemma \ref{lemma:diffNorm1}, 
\begin{align*}
\lim_{t\downarrow 0^+}\frac{|\abf^T\x(t)| - |a_1|}{t} & = \lim_{t\downarrow 0^+} \frac{a_1x_1(t) - a_1}{t} + \lim_{t\downarrow 0^+} \sum_{i=2}^m\frac{a_ix_i(t)}{t} \\
& = a_1\lim_{t\downarrow 0^+} \dot{x}_1(t) + \sum_{i=2}^m a_i \lim_{t\downarrow 0^+} \dot{x}_i(t).
\end{align*}
\end{proof}

\begin{lemma}
\label{lemma:diffNorm4}
Let $\abf = (a_1,...,a_m)^T\in\mathbb R^m$ and $\x(t) = (x_1(t),...,x_m(t))^T\in \mathbb R^m$ be an $m$-dimensional function on $[0,\epsilon)$ such that: (1) $x_i(0)=0$ for all $i = 1,...,m$; (2) The derivative $\dot{x}_i(t)$ exists for all $t \in (0,\epsilon)$. We have
\[
\lim_{t\downarrow 0^+}\frac{|\abf^T\x(t)|}{t} = \left|\sum_{i=1}^m a_i \lim_{t\downarrow 0^+} \dot{x}_i(t)\right|.
\]
\end{lemma}
\begin{proof}
The proof is similar to that of Lemma \ref{lemma:diffNorm}.
\end{proof}

\bibliographystyle{apalike}
\bibliography{reference}

\end{document}